\documentclass[final]{cvpr}

\usepackage{times}
\usepackage{epsfig}
\usepackage{amsmath}
\usepackage{amssymb}
\usepackage{bm}

\usepackage{url}
\usepackage{multirow}
\usepackage{color}
					
\usepackage[dvipsnames]{xcolor}

\definecolor{mygrey}{rgb}{0.3, 0.3, 0.3}
\definecolor{airforceblue}{rgb}{0.36, 0.54, 0.66}

\newcommand{\bz}{\bm{z}}

\newcommand{\bt}{\bm{t}}
\newcommand{\bv}{\bm{v}}

\usepackage{mathtools}
\usepackage{algpseudocode}
\usepackage{algorithm}
\usepackage{algorithmicx}
\newcommand{\commentT}[1]{{}}

\DeclareMathOperator*{\argmin}{arg\,min}

\usepackage{textcomp}
\usepackage{filecontents}
\usepackage{microtype}
\usepackage{float}
\usepackage{adjustbox}
\usepackage{booktabs,makecell,tabularx}
\usepackage{url}
\usepackage{booktabs}

\usepackage{graphicx}
\usepackage[font=small]{caption}
\usepackage{subcaption}

\newcolumntype{C}[1]{>{\centering\arraybackslash}p{#1}}
\newcolumntype{L}{>{\raggedright\arraybackslash}X}
\usepackage{siunitx}
\usepackage[utf8]{inputenc}

\usepackage{etoolbox}
\usepackage{xparse}
\makeatletter

\usepackage{array}
\usepackage{makecell}

\usepackage{diagbox}

\usepackage{footnote}
\makesavenoteenv{tabular}

\usepackage{environ}

\usepackage{pifont}% http://ctan.org/pkg/pifont
\newcommand{\cmark}{\ding{51}}%
\newcommand{\xmark}{\ding{55}}%

\usepackage{titling}

\usepackage[pagebackref=true,breaklinks=true,colorlinks,bookmarks=false]{hyperref}

\newcommand{\spm}[1]{\color{mygrey}\tiny $\pm{#1}$} 
\newcommand{\tc}[2]{\shortstack{{#1} \\ {\spm{#2}}}}
\newcommand{\ttc}[2]{\thead{\tc{\footnotesize{#1}}{#2}}}

\usepackage{amsthm}

\newtheorem{proposition}{Proposition}

\newtheorem{pcorollary}{Corollary}[proposition]

\usepackage{color,soul}
\usepackage[framemethod=default]{mdframed}
\usepackage[normalem]{ulem}

\NewEnviron{tight_eq}{  
\begin{equation}
    \setlength{\belowdisplayskip}{2pt}
\setlength{\abovedisplayskip}{2pt} 
\BODY
    \end{equation}}

\newcommand{\R}{\mathbb{R}}
\newcommand{\E}{\mathbb{E}}
\newcommand{\cD}{\mathcal{D}}

\newcommand{\cX}{\mathcal{X}}
\newcommand{\cZ}{\mathcal{Z}}

\usepackage{enumitem}
\newcommand{\mypar}[1]{\vspace{-3mm}\paragraph{#1}}

\begin{document}

%%%%%%%%% TITLE

\title{Multiclass non-Adversarial Image Synthesis \\ with Application to Classification from Very Small Sample}

\author{Itamar Winter\\
The Hebrew University of Jerusalem\\
{\tt\small itamar.winter@mail.huji.ac.il}
\and
Daphna Weinshall\\
The Hebrew University of Jerusalem\\
{\tt\small daphna@cs.huji.ac.il}
}

\date{}

\maketitle

%%%%%%%%% ABSTRACT
\begin{abstract}
    The generation of synthetic images is currently being dominated by Generative Adversarial Networks (GANs). Despite their outstanding success in generating realistic looking images, they still suffer from major drawbacks, including an unstable and highly sensitive training procedure, mode-collapse and mode-mixture, and dependency on large training sets. In this work we present a novel non-adversarial generative method - Clustered Optimization of LAtent space (COLA), which overcomes some of the limitations of GANs, and outperforms GANs when training data is scarce. In the full data regime, our method is capable of generating diverse multi-class images with no supervision, surpassing previous non-adversarial methods in terms of image quality and diversity. In the small-data regime, where only a small sample of labeled images is available for training with no access to additional unlabeled data, our results surpass state-of-the-art GAN models trained on the same amount of data. Finally, when utilizing our model to augment small datasets, we surpass the state-of-the-art performance in small-sample classification tasks on challenging datasets, including CIFAR-10, CIFAR-100, STL-10 and Tiny-ImageNet. A theoretical analysis supporting the essence of the method is presented.
\end{abstract}

%%%%%%%%% BODY TEXT
% -----------------------------------------------------------------------
\section{Introduction}

\begin{figure}[ht]
\captionsetup[subfigure]{labelformat=empty}
     \centering
     \begin{subfigure}[b]{0.225\textwidth}
     \caption{COLA}
         \centering
         \includegraphics[]{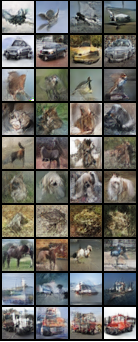}
         
         \label{fig:clog}
     \end{subfigure}
     \begin{subfigure}[b]{0.225\textwidth}
        \caption{GLO}
         \centering
         \includegraphics[]{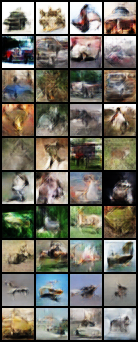}
         
         \label{fig:glo}
     \end{subfigure}
     \hfill
        \caption{Training on CIFAR-10 with no labels: the images generated by our method (left), which imposes semantic structure on the latent space, are superior to the alternative method (right). Each row holds a random sample from a distinct object class.}
        \label{fig:full_data_qualitative}
        \vspace{-0.5cm}
\end{figure}

Generative image modeling is a long-standing challenge in computer vision. Unconditional generative models aim at learning the underlying distribution of the data using a finite training set, and synthesizing new samples from the learned distribution. Recently, deep generative models have shown remarkable results in synthesizing high-fidelity and diverse images. Most notably, Generative Adversarial Networks (GANs) \cite{goodfellow2014generative} have been extensively used in classic computer vision tasks such as image generation, image restoration and domain translation, alongside traditional learning tasks such as data augmentation  \cite{frid2018synthetic} and clustering \cite{ben2018gaussian, mukherjee2019clustergan}.

Since their inception, the unsupervised training of GANs achieved effective models able to produce natural-looking images, while relying on a simple and easily modified framework. Nevertheless, and despite numerous efforts for improvement, GANs still exhibit some critical drawbacks that arise from the adversarial nature of the optimization. These include: (i) an unstable training procedure, that is highly sensitive to the choice of initialization, architecture and hyper-parameters; (ii) often the learned distribution suffers from mode-collapse, in which  only a subset of the real distribution is covered by the model, or mode-mixture, where different modes are mixed with each other. These problems are amplified when training data is scarce \cite{xu2018empirical}.

These drawbacks have motivated research into non-adversarial alternatives such as Variational Auto Encoders (VAE) \cite{kingma2013auto}  and Generative Latent Optimization (GLO) \cite{bojanowski2017optimizing}.  VAEs learn generative deep models that include a representation layer defining the model's latent space, where both the prior and posterior distributions over the latent space are approximated by parametric Gaussian distributions. GLO learns a non parametric prior over the latent space in unison with the generative model. Although the VAE framework stands on solid theoretical foundations, VAEs generally do not generate sharp images, partially due to the restrictive parametric assumptions that are enforced. GLO, on the other hand, imposes hardly any limitation on the learned distribution over the latent space, which is guided only by the reconstruction performance of the model. Alas, as a result the structure of the latent  space holds no semantic information, and cannot be effectively sampled from. These limitations are aggravated when dealing with multi-modal distributed data, as is typically the case with multi-class data. 

Broadly speaking, most contemporary generative models rely on common and often implicit assumptions: (i) the Manifold Hypothesis, which assumes that real-world high-dimensional data lie on low-dimensional manifolds embedded within the high-dimensional space; (ii) that there exists a mapping from a low dimension latent space onto the real data manifold; (iii) that this latent space can be approximated by a single Gaussian distribution (such is the latent prior distribution in most variants of GANs, VAEs, and GLO); and (iv) that the generative model is capable of learning the assumed mapping. While these assumptions may hold true when trying to learn from data that resides on a single manifold, it is impossible for a continuous mapping (\ie CNN generator) to effectively map a connected latent space onto a disconnected data manifold of a multi-class distribution \cite{kelley2017general}. 

In this work, we seek to overcome both the inherent drawbacks of the GAN framework and the deficiency of the uni-modal Gaussian prior in modeling the latent space. Thus, in Section~\ref{sec:method} we propose an unsupervised non-adversarial generative model, that optimizes the latent space by fitting a multi-modal data distribution. Unlike GLO, our latent space preserves semantical information about the data, while the multi-modal distribution allows for the efficient and direct sampling of new data. As will be shown in Section~\ref{sec:fulldata}, the distribution over the latent space that is learnt by our model captures semantic properties of the data. As a result, our model is capable of generating better images in terms of image quality, diversity and discriminability. In Section~\ref{sec:theory} we provide some theoretical justification for our method.

Expanding to domains where GANs do not excel, our model is designed to be applicable for downstream tasks where training data is scarce. The task of learning from small sample is usually tackled with the aid of external data or prior knowledge. While transfer-based techniques work well when the source and target domains share distributional similarities, it is not at all the case when the target data comes from a considerably different domain (such as medical imaging) \cite{raghu2019transfusion, litjens2017survey}. Furthermore, gaining access to large labeled datasets may not always be possible due to legal and ethical considerations. In contrast, here we tackle the small-sample classification task where\textbf{ no prior knowledge or external data is present}. In this setting, the training algorithm may get as few as 5 images per class, having access to no additional labeled or unlabeled data. This constitutes a very challenging task. In Section~\ref{sec:small-sample} we show that, when using our model to augment the real data, we are able to advance the state-of-the-art and achieve top performance in small sample classification tasks. 

Our main contributions are as follows:
\begin{itemize}[itemsep=1pt, leftmargin=5mm]
    \item Introduce a novel unsupervised non-adversarial generative model capable of synthesizing diverse discriminable images from multi-class distributions (Section~\ref{sec:method}).

        \item Provide sufficient conditions and a simplified theoretical framework, under which our method can be beneficial in approximating under-sampled distributions (Section~\ref{sec:theory}). 
    \item Demonstrate superior image synthesis capabilities when training data is scarce, as compared to state-of-the-art GAN models (Section~\ref{sec:fulldata}).
    \item  Apply our model to small-sample classification tasks, surpassing all previous work in this domain (Section~\ref{sec:small-sample}).
\end{itemize}

% ---------------------------------------------------------------
\section{Recent Related Work}

\paragraph{Modeling disconnected data manifolds.}

The issue with mapping a connected latent space onto a disconnected data manifold was mainly addressed in the context of overcoming mode-collapse in GANs. \cite{che2016mode, larsen2016autoencoding, donahue2016adversarial, srivastava2017veegan} use an encoder to match the latent code with the data distribution. While the latent representation of these methods is optimized via a reconstruction loss of the decoder, our method learns a representation that holds semantical information.  
	  
Another line of work \cite{khayatkhoei2018disconnected, arora2017generalization, hoang2018mgan, ghosh2018multi} uses multiple generators in order to cover all the modes in the data, while \cite{an2019ae, hoshen2019non} learn a mapping from a normally distributed noise to an optimized latent structure in a non-adversarial framework. Other works use a GMM prior over the latent space in VAEs \cite{dilokthanakul2016deep, shu2016stochastic} and GANs \cite{ben2018gaussian}. Finally, \cite{chen2016infogan} combines discrete and continuous latent factors to learn a disentangled representation of the data. 

\mypar{Data Embedding and Feature Learning.}
Learning a meaningful low-dimensional embedding for high-dimensional data has been significantly improved by advances in deep neural networks and self-supervised learning.  Thus, \cite{chang2017deep, haeusser2018associative, ji2019invariant} all harness the large capacity of deep neural networks to learn efficient clustered representations of natural images. In a related line of work, \textit{self-supervised learning} involves the learning of meaningful visual features from a pretext task using labels that are produced from the data itself with no direct supervision. These include jigsaw puzzle solving \cite{noroozi2016unsupervised}, predicting positions of patches in an image \cite{doersch2015unsupervised}, and predicting image rotations (\emph{'RotNet'}) \cite{gidaris2018unsupervised}. In this work, we learn a clustered embedding of the data and a self-supervised pretext task en masse, which greatly improves the quality of the learned representation. 

\mypar{Learning from small sample.}
\label{ss related work}
Classification from small sample, with no prior knowledge or access to external data, has been chiefly approached by attempting to augment the sample into a sufficiently large training set. Thus \textbf{DADA} \cite{zhang2019dada} adapts a GAN model for this purpose, \textbf{TANDA} \cite{ratner2017learning} uses GANs to learn generic data augmentations composed of pre-defined transformations using large unlabeled data, DHN \cite{oyallon2017scaling} uses a hybrid network that incorporates learnable weights with a scattering network of predefined wavelets, and \textbf{CFVAE-DHN} \cite{lin2020efficient} augments the latent variables of a VAE, which in turn generates additional data that is classified using a DHN. Likewise, \cite{barz2020deep} promotes the use of the cosine-loss, and \cite{brigato2020close} promotes low-complexity networks.

Other methods usually incorporate some form of \textit{transfer learning} \cite{zhuang2020comprehensive}, where parameters that are learned in a source domain are transferred and utilized in a different target domain.  Transfer learning from large datasets to smaller ones has also been investigated in generative models \cite{wang2018low, wang2020minegan, noguchi2019image, wang2018transferring}. One recent prominent research paradigm in this context is termed \textit{few-shot learning} \cite{wang2020generalizing}. 
We consider this paradigm to be an instance of transfer (or meta) learning and not strict classification from small sample, as it requires access to an external labeled dataset.

%-------------------------------------------------------------------------
\section{Proposed Method}
\label{sec:method}

\subsection{COLA: Unsupervised Algorithm}
\label{method:unsupervised}

Our method, Clustered Optimization of LAtent space (COLA), is an unsupervised method which learns a generative model for the synthesis of images. The method is designed to cope with a small training set of natural images, portraying distinct object categories. It involves three steps, the first two of which are illustrated in Fig.~\ref{fig:model}. Our code is supplied in the supplementary materials.

\mypar{Step I: Clustering the latent space.} 
The goal here is to deliver a mapping from the data space to a latent space, while clustering the mapped points into compact $K$ clusters, see illustration in Fig.~\ref{fig:model}. To this end we train a deep convolutional network $E_\theta$, which maps each data point to a fixed low-dimensional target point in the latent space - the unit sphere in $\mathbb{R}^K$. The target is sampled from a pre-defined distribution over the latent space. 

\begin{figure}[t]
\begin{center}

  \includegraphics[width=\linewidth]{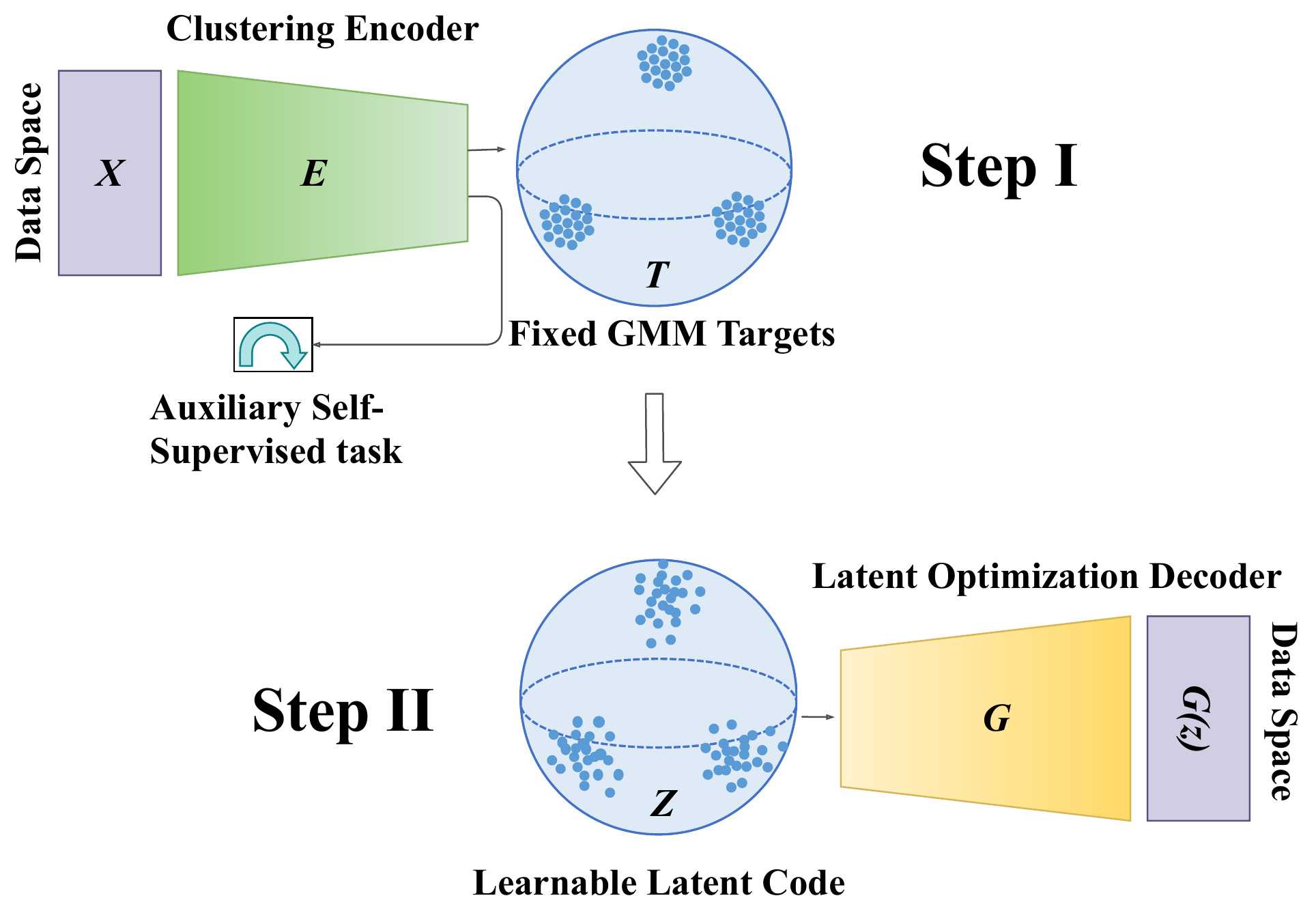}
\end{center}
\vspace{-0.3cm}
  \caption{Illustration of our model: In step I images are mapped to fixed low-dimensional targets $T$. In step II these targets form a latent space $Z$ that is trained in conjunction with the generator parameters to  reconstruct the original image.}
\label{fig:model}
\vspace{-0.5cm}
\end{figure} 

Specifically, given an unlabeled dataset $\mathcal{X} = \{x_i\}_{i=1}^N $, the model is initialized with some random assignment  $\{(x_i, \bt_i)\}_{i=1}^N$, where each $\bt_i \in \mathbb{R}^K$ is sampled from a GMM distribution with K-components, and normalized to length 1. Training involves the minimization of
\begin{equation}
\label{eq:step1}
        \|E_\theta(x_i) - \bt_{\pi(i)}\|_2^2   
\end{equation}
over the assignment $\{(x_i, \bt_{\pi(i)})\}_{i=1}^N$ and parameters $\theta$. 

This optimization problem is solved with SGD, and involves two steps per mini-batch. First, the sample  $\{x_i\}_{i\in b}$ is mapped onto the latent space. The assignment problem for $\{(x_i, \bt_{\pi(i)})\ | i \in b\}$ is solved using the Hungarian Algorithm \cite{kuhn1955hungarian} applied to the following problem:
    \begin{equation}
    \label{eqn:1}
    \pi^* =  \argmin_{\pi : b \leftrightarrow b} \sum_{i\in b} \|E_\theta(x_i) - \bt_{\pi(i)}\|_2^2  
    \end{equation} 
Subsequently $\pi^*$ is inserted into Eq. \ref{eq:step1}, which is then optimized  w.r.t $~\theta$.  A full formulation and implementation details can be found in Alg.~\ref{supp: clustering pseudocode} and Section~\ref{supp: step1 implementation} respectively in the Appendix. Alg.~\ref{supp: clustering pseudocode} is enhanced with self-supervision based on the auxiliary \emph{'RotNet'} task \cite{gidaris2018unsupervised}, and consistency regularization where augmented images are mapped to the same cluster.

The output of this model constitutes a latent space, where the representations of semantically similar images reside in proximity, and images from distinct classes are located further apart. This representation is used to initialize the latent space of the generative model in step II. To simplify the presentation, henceforth we let $\bt_i = \bt_{\pi^*(i)}$ denote the final target associated with $x_i$.

\begin{figure*}[bht]
\begin{center}
\vspace{-0.1cm}
\includegraphics[width=0.9\linewidth]{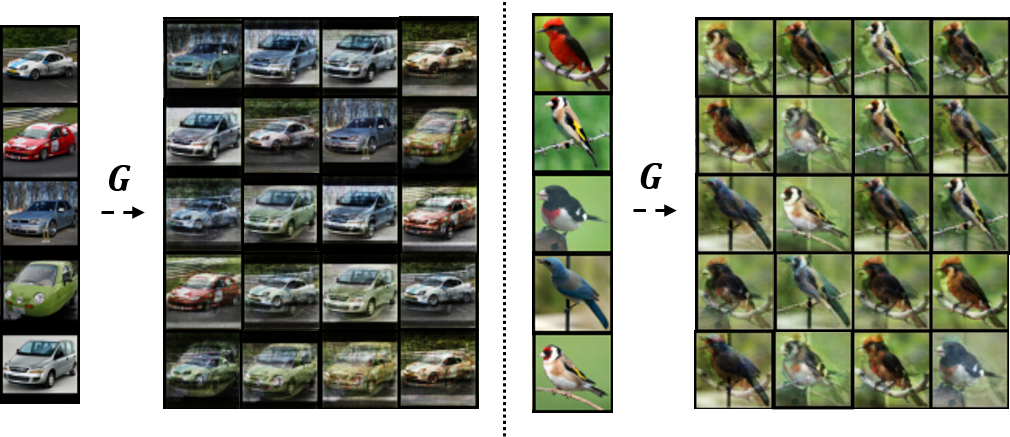}
\end{center} 
\vspace{-0.5cm}
\caption{Synthetic images generated by our model when trained on STL-10 with 5 images per class and no external data. Real images are shown on the left, synthetic images are shown on the right.}
\label{fig:cars_birds}
\vspace{-0.35cm}
\end{figure*}

\mypar{Step II: Image generation.} \label{step 2}
Given a matching between data points $\mathcal{X}$ and targets $\mathcal{T}$ -  $\{(x_i,\bt_i)\}_{i=1}^N$, a latent code $\mathcal{Z} = \{\bz_i\}_{i=1}^N$ is constructed such that
\begin{equation*}
\bz_i =  (\frac{\bt_i}{\|\bt_i\|_2}, \frac{\bv_i}{\|\bv_i\|_2}) \in \mathbb{R}^{K+d}
\end{equation*}
Above $\bv_i \sim \mathcal{N}(\vec{0}, \sigma I_{d \times d})$ denotes an additional source of variation, and $\bt_i$ denotes the class-component of the code.

The parameters of a CNN generator function $G_\theta : \mathcal{Z} \longrightarrow \mathcal{X}$ are optimized in conjunction with the learnable representation vectors $\mathcal{Z}$, as illustrated in Fig.~\ref{fig:model}. The optimization problem is defined as:
\begin{tight_eq}
\label{eqn:2}
\min_{\theta,\mathcal{Z}} \frac{1}{N} \sum_{i=1}^N \mathcal{L}_{rec}\big(G_\theta(\bz_i), x_i\big) \quad s.t. \quad \|\bz_i\|_2 = 1
\end{tight_eq}
where $\mathcal{L}_{rec}$ denotes the reconstruction loss between the original image $x_i$ and the image generated by the model $G_\theta(z_i)$. 

As shown in \cite{hoshen2019non}, the best image quality for this kind of models may be obtained when $\mathcal{L}_{rec}$ is realized with the perceptual loss \cite{johnson2016perceptual}:
\begin{tight_eq}
\label{eqn:perceptual_loss}
\mathcal{L}_{vgg}(x,x') =  |x - x'|\: + \sum_{layers:\: i}^k |l_i(x) - l_i(x')| 
\end{tight_eq}

In (\ref{eqn:perceptual_loss}) $l_i$ denotes the perceptual layer in a pre-trained VGG network \cite{simonyan2014very}. Nevertheless, since external data cannot be used in the small sample scenario adopted here, $\mathcal{L}_{rec}$ is realized in our method with the Laplacian Pyramid loss:
\begin{tight_eq}
\label{eqn:lap_loss}
     \mathcal{L}_{lap}(x,x') =  |x - x'|\: +\gamma \sum_{i}^k 2^{-2i} |L_i(x) - L_i(x')| 
\end{tight_eq}

In (\ref{eqn:lap_loss})  $L_i(x)$ denotes the i-th level of the Laplacian pyramid representation of $x$ \cite{ling2006diffusion}. The sum of differences is weighted to preserve the high-frequencies of the original image. The components of the representation vectors are normalized after each epoch to length 1, projecting them back to the unit spheres in $\mathbb{R}^K , \mathbb{R}^d$ respectively. 

This step is summarized below in Alg.~\ref{step 2 algo}. Full implementation details are presented in Appendix~\ref{supp: step2 implementation}.

\mypar{Step III: posterior distribution over the latent space.}

After training, a posterior distribution over the latent space is obtained by fitting a unique multivariate Gaussian to each cluster in the latent space. Sampling is then performed from the uniform mixture of these Gaussian distributions. 

    \begin{algorithm}[ht] 
             \caption{\textbf{: Training the Generative Model}} 
               \label{step 2 algo}

            \begin{flushleft}
                \textbf{INPUT:}  \\
                \quad matched pairs $\{(x_i, \bt_i)\}_{i=1}^N \subset X \times [0,1]^K$ from step I\\

                 \quad $G_\theta$ - CNN Generator with parameters $\theta$\\
                \quad  $\lambda^\theta_e, \lambda^z_e$ -  learning rate at epoch e of $\theta, Z$ \\
                \quad $ \sigma $ - pre-defined latent std 

            \end{flushleft}
            \begin{algorithmic}
            \For {\textit{i=1...N}} 
            \Comment{initialize latent space}
           \State sample $\bv_i \sim \mathcal{N}(\vec{0}, \sigma I_{d \times d})$
            \State $\bz_i \leftarrow (\frac{\bt_i}{\|\bt_i\|_2}, \frac{\bv_i}{\|\bv_i\|_2}) \in \mathbb{R}^{K+d}$
            \EndFor
            
            \For {\textit{e=1...epochs}}
            \For {\textit{i=1...iters}}
            \State sample batch $\{(x_i, \bz_i) | i \in B\}$
           
            $$ \mathcal{L}_B= \frac{1}{|B|}\sum_{i\in B}\mathcal{L}_{rec}(x_i,G_\theta(\bz_i)) $$  
            \State $\theta \leftarrow   \theta - \lambda^\theta_e (\nabla_\theta \mathcal{L}_B) $
            \State $\bz \leftarrow   \bz - \lambda^z_e (\nabla_z \mathcal{L}_B) $
            \EndFor
            
            \State  \quad\quad $\bt\leftarrow\bz_{[1:K]},\quad\quad\bv\leftarrow\bz_{[K+1:K+d]} $
            \State $\forall_i \quad \bz_i \leftarrow (\frac{\bt_i}{\|\bt_i\|_2}, \frac{\bv_i}{\|\bv_i\|_2}) $\Comment{Normalize inputs}\\
            \EndFor
        \end{algorithmic}
    \end{algorithm}

\subsection{sCOLA: Supervised Algorithm} 
\label{supervised}

In the supervised framework, we have a labeled dataset with $K$ classes  $\mathcal{X} = \{(x_1,y_1), (x_2, y_2), \cdots , (x_N, y_N)\}$, where $y_i \in [K]$ denotes the class label of $x_i$, and $\bm{e}_y^i$ denotes the one-hot representation of the labels. The supervised version of our method, \emph{sCOLA} includes steps II and III of COLA. The clustering in step I is replaced by the supervision labels from the training data, where each $\bt_i$ is replaced by the corresponding $\bm{e}_y^i$. Fig.~\ref{fig:cars_birds} shows images generated by our model with only 5 training examples per class.

% -----------------------------------------------------------------------
\section{Theoretical Analysis}
\label{sec:theory}

Stripped off its technical details, the method in Section~\ref{sec:method} essentially learns a noisy surrogate distribution $Z$ to approximate the real data distribution $X$ and generate new data. In this paper, our ultimate goal is not to generate new high quality data, but rather to estimate some function $f: X\longrightarrow \Omega$ from a sample of $X$. When $X$ denotes data sampled from $K$ discrete classes, a multi-class classifier is such a function whose codomain is either $[K]$ or $\R^K$. If the sample of $X$ is too small, the surrogate distribution $Z$ can be used to generate more data and improve the estimation of $f$. The analysis below identifies sufficient conditions on the respective sample sizes, such that improvement can indeed be guaranteed.

\mypar{Notations.}
Assume an i.i.d. sample of random variable pairs - $\{X_i,Y_i\}_{i=1}^N$, where $X_i/_{Y_i=k}\overset{iid}{\sim} \cD_{k}$ and $\cD_k$ denotes the class conditional distributions of variable $X$. Let $\cX_k$ denote the conditional sub-sample of datapoints from class $k$: $\cX_k=\{X_{i_j},Y_{i_j}/_{Y_{i_j}=k}\}_{j=1}^{m_k}$, where $\sum_{k=1}^K m_k=N$. 

For simplicity, we will assume in our analysis that $f$ depends only on the expected value of the conditional distributions $\{\cD_k\}_{k=1}^K$, denoted $\mu_k$. Let $\tilde \mu_k(X)$ denote an estimator of $\mu_k$ from an iid sample of random variable $X$. Our task is to obtain a set of good estimators $\{\tilde \mu_k\}_{k=1}^K$. In order to simplify the notations, we shall henceforth drop the class index k, with the understanding that the following analysis does not depend on $k$.

In accordance, let $\cX^m$ denote an \emph{iid} sample of size $m$ from the real conditional distribution $X$ of some class $k$, and $\cZ^n$ denote an \emph{iid} sample of size $n$ from the class surrogate distribution $Z$. Let $\mu_x=\mu$ denote the expected value of $X$, and $\mu_z$ denote the expected value of $Z$, where $|\mu_x-\mu_z|=d$. Let $\bar\cX^m$ and $\bar\cZ^n$ denote the population means of the two samples respectively. Recall that $Var[\bar\cX^m ] =\frac{Var[X]}{m}$ and $Var[\bar\cZ^n ] =\frac{Var[Z]}{n}$.

As customary, we use the population mean of each sample to estimate the unknown distribution's mean $\mu$. Accordingly:
\begin{equation}
\begin{split}
\tilde \mu(X) &= \bar\cX^m \\
\tilde \mu(Z) &= \bar\cZ^n 
\end{split}
\end{equation}
The error of the two estimators is measured as follows:
\begin{equation}
\begin{split}
Err(X) &= (\bar\cX^m -\mu)^2 \\
Err(Z) &= (\bar\cZ^n -\mu)^2 
\end{split}
\end{equation}

\begin{proposition}   
If $Var[X] > md^2$, then
\begin{equation*}
n \geq \frac{m Var[Z]}{Var[X] -md^2} ~\implies~  \E[Err(Z)] \leq \E[Err(X)]
\end{equation*}
\end{proposition}

\begin{proof}
\begin{equation*}
\begin{split}
\E[Err(Z)] &=  \mathbb{E}[(\bar\cZ^n - \mu)^2] ~~~~~~~~~~~~~~~~{\color{airforceblue}(\textrm{use}~\mu= \mu_x)}\\
&\leq \mathbb{E}[(\bar\cZ^n - \mu_z)^2] + d^2  = \frac{Var[Z]}{n} + d^2
\end{split}
\end{equation*}
As long as $Var[X] > md^2$
\begin{equation*}
n \geq \frac{m Var[Z]}{Var[X] -md^2} ~\implies~  \frac{Var[Z]}{n} + d^2 \leq \frac{Var[X]}{m} 
\end{equation*}
and therefore
\begin{equation*}
\E[Err(Z)] \leq Var[\bar\cX^m] = \E[Err(X)]
\qedhere
\end{equation*}
     
\end{proof}
\begin{pcorollary}
For each class $k$, if the sample of the surrogate random variable $Z$ is sufficiently large
\begin{equation*}
n \geq \frac{m Var[Z]}{Var[X] -md^2}
\end{equation*}
then the estimator of classifier $f$ obtained from $\cZ^n$ is more accurate than the estimator obtained from  $\cX^m$.
\end{pcorollary}

\begin{proposition}
Assume that $Pr[0 \leq X, Z \leq 1] = 1$, which can be achieved by dataset normalizing. Then $\forall \epsilon > d$, if $n\geq m (\frac{\epsilon}{\epsilon-d})^2$, then the bound obtained by the Hoeffding's inequality on $Pr(|Err(Z)| \geq \epsilon)$ is tighter than the corresponding bound on $Pr(|Err(X)| \geq \epsilon)$.
\end{proposition}

\begin{proof}
We invoke the Hoeffding's inequality: 
\begin{equation*}
 Pr ( |\bar\cX^m - \mu |> \epsilon) \leq 2e^{-2m\epsilon^2} 
\end{equation*}
and note that 
\begin{equation*}
|Err(Z)| \leq |\bar\cZ^n - \mu_z| + |\mu_x -  \mu_z| = |\bar\cZ^n - \mu_z| + d 
\end{equation*}
It follows that
\begin{equation*}
\begin{split}
 Pr(|Err(Z)| \geq \epsilon) &\leq Pr(|\bar\cZ^n - \mu_z| + d \geq \epsilon ) \\ 
 &= Pr(|\bar\cZ^n - \mu_z| \geq \epsilon - d) \\
 &\leq  2e^{-2n(\epsilon-d)^2}:=B(Z) \\
Pr(|Err(X)| \geq \epsilon) &= Pr(|\bar\cX^m - \mu_x| \geq \epsilon) \\
&\leq  2e^{-2m\epsilon^2} :=B(X)
\end{split}
\end{equation*}

Finally
\begin{equation*}
n\geq m (\frac{\epsilon}{\epsilon-d})^2  \implies  B(Z) \leq B(X)
\qedhere
\end{equation*} 

\end{proof}

\begin{pcorollary}
For each class $k$, if the sample from the surrogate random variable $Z$ is sufficiently large
\begin{equation*}
n  \geq  m (\frac{\epsilon}{\epsilon-d})^2 ,
\end{equation*}
then the estimator of classifier $f$ obtained from $\cZ^n$ is more confident than the estimator obtained from  $\cX^m$.
\end{pcorollary}

% -----------------------------------------------------------------------

\section{Image generation, Large and Small Sample}
\label{sec:fulldata}

We shall now demonstrate the capability of our model to produce diverse and discriminable images, employing evaluation metrics that quantify these attributes. Firstly, we compare our model with competitive conditional GAN models that use large and computationally heavy architectures. While these models maintain superiority on  large datasets, this dominance diminishes as the sample size drops.  Secondly,  we show that our unsupervised variant surpasses other unsupervised  generative adversarial models using the same architecture. Lastly, we show that our model consistently outperforms other non-adversarial methods in terms of image quality and diversity, regardless of sample size.

\subsection{Methodology}

\paragraph{Datasets} 
The datasets we use are included in Table~\ref{table:datasets}.
\mypar{Evaluation scores.} 
Designing meaningful quantitative evaluation measures for generative models is a challenging ongoing research area. Presently two scores seem to dominate the field: the  Inception Score \cite{salimans2016improved}, and the Fr\'echet Inception Distance (FID) \cite{heusel2017gans}. Noting that the Inception score does not take into account the real data distribution and cannot capture intra-class diversity, we will not be using this metric in our evaluation.

FID compares the statistics of activations in the penultimate layer of the Inception network (trained on  'ImageNet') between real and generated images, computing the distance between the uni-modal Gaussian distributions that best fit the activation patterns. This score has two major drawbacks: (i) \label{problem 1} it captures image quality and diversity on a single scale, and therefore cannot distinguish between the two factors;  (ii) \label{problem 2} it is based on the Inception network that has been trained with 1,000 classes of 'ImageNet', and may not be suitable for all datasets \cite{barratt2018note}. In Appendix~\ref{supp: fid inadequacy} we show that the FID score also fails to reveal intra-class diversity, making it less useful for multi-class datasets (see also  \cite{liu2018improved}). Implementation details for the FID score used in our experiments can be found in Appendix~\ref{supp: fid implementation}.

Given the problems discussed above, we seek an additional score that can reliably measure how well the generated images fit the true distribution of the data. More importantly,  considering that generative models are commonly used in down-stream tasks, we seek a score that can measure the usefulness of the model generations in such tasks. To this end we adopt the scores proposed in \cite{shmelkov2018good} ('GAN-Train') and  \cite{ravuri2019classification} ('CAS'), which are based on training a classification network on the generated images, and evaluating it on real images. The classification accuracy of this network forms an implicit measure of the recall and precision of the generated dataset, since it can only achieve a high score if the synthetic data is sufficiently diverse and discriminable. In our experiments, we follow the protocol defined in \cite{ravuri2019classification}.

\mypar{Generative methods used for comparisons.} 

We compare our model against state-of-the-art generative models, one adversarial model based on the GAN framework, and a second non-adversarial method:
\setlist{nolistsep}
\begin{enumerate}[noitemsep]
    \item Adversarial \textbf{ CGAN-PD} \cite{miyato2018cgans}: a conditional GAN with Projection-Discriminator, trained and implemented in accordance with  \cite{lee2020mimicry}.

    \item Non-adversarial \textbf{GLO} \cite{bojanowski2017optimizing}: the original model augmented with the superior perceptual loss from Eq.~\ref{eqn:perceptual_loss}. Similarly to step III above, after training we fit a Gaussian Mixture Model to the learned latent space.
\end{enumerate}
Implementation details can be found in Appendix~\ref{supp: step2 implementation}.

\subsection{Results} 
\label{sec: full_data results}

\paragraph{Unsupervised.}
In the unsupervised scenario, we compare our model to the baseline GLO model, see Fig.~\ref{fig:unsupervised}. Clearly our model outperforms GLO on all datasets and metrics, and produces significantly better looking images as demonstrated in Fig.~\ref{fig:full_data_qualitative}. Furthermore, we recall
that different GAN models can reach similar FID scores if given a high enough computational budget \cite{lucic2018gans}. We therefore adopt the fair comparison protocol proposed in \cite{lucic2018gans}, where the architectures of all the models are fixed to the one used in 'InfoGAN' \cite{chen2016infogan}, and all models possess the same computational budget for hyper-parameter search. In this protocol, our method outperforms all GAN variants and is on par with the state-of-the-art non-adversarial methods, see Fig.~\ref{fig:lucic_gans}. 

\begin{figure}[hbt]

   \includegraphics[width=\linewidth]{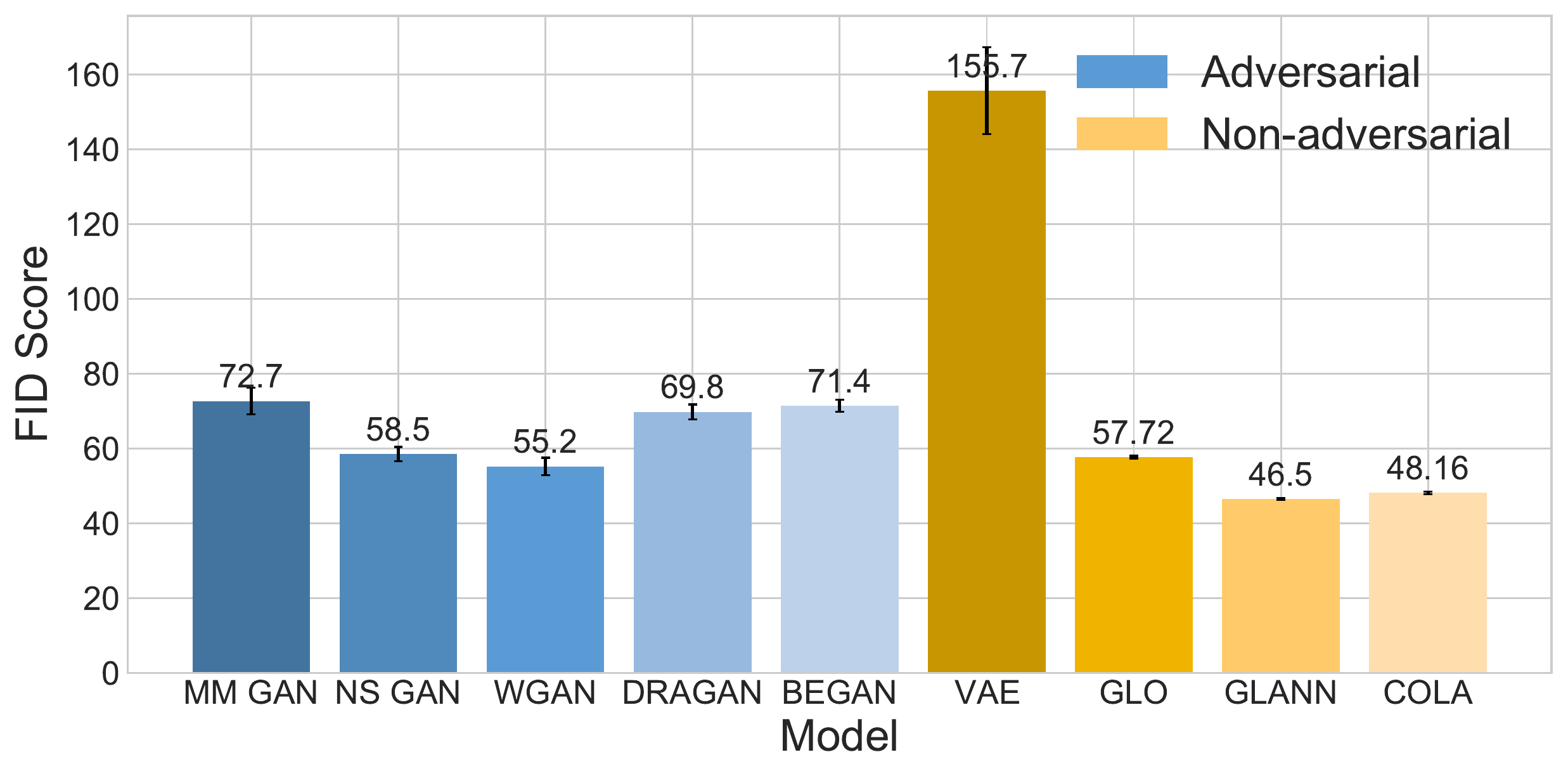}
   \vspace{-0.5cm}
   \caption{FID score computed for CIFAR-10, when all models share the same architecture of 'InfoGAN'~\cite{chen2016infogan}. Unlike all other models in this comparison, our method allows for the sampling of images from different individual classes.}
\label{fig:lucic_gans}
\vspace{-0.35cm}
\end{figure}

\begin{figure}[hbt]
	\centering
		\includegraphics[width=0.49\linewidth]{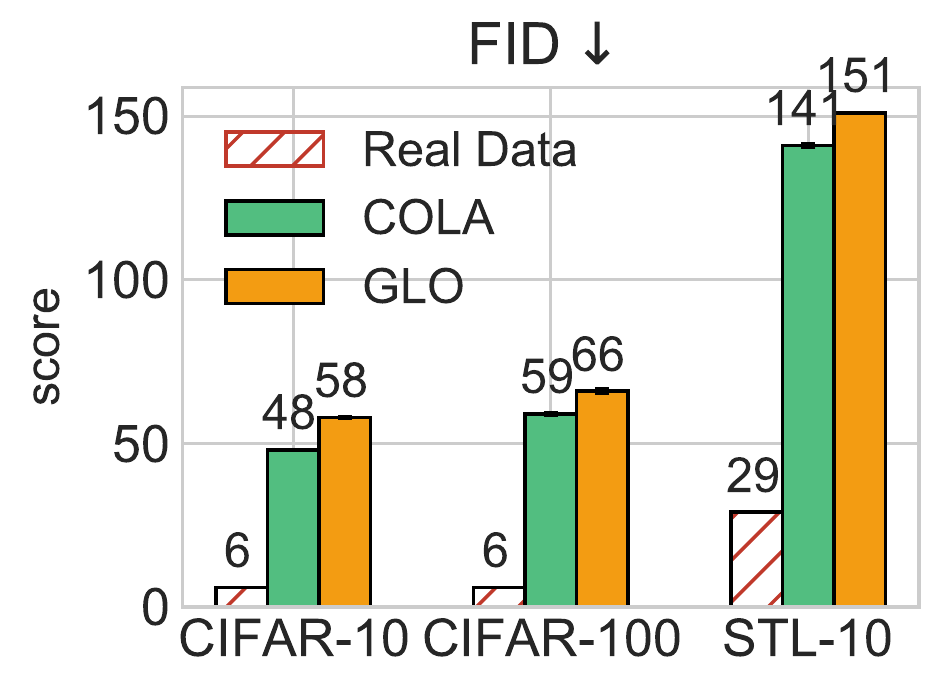}
		\includegraphics[width=0.49\linewidth]{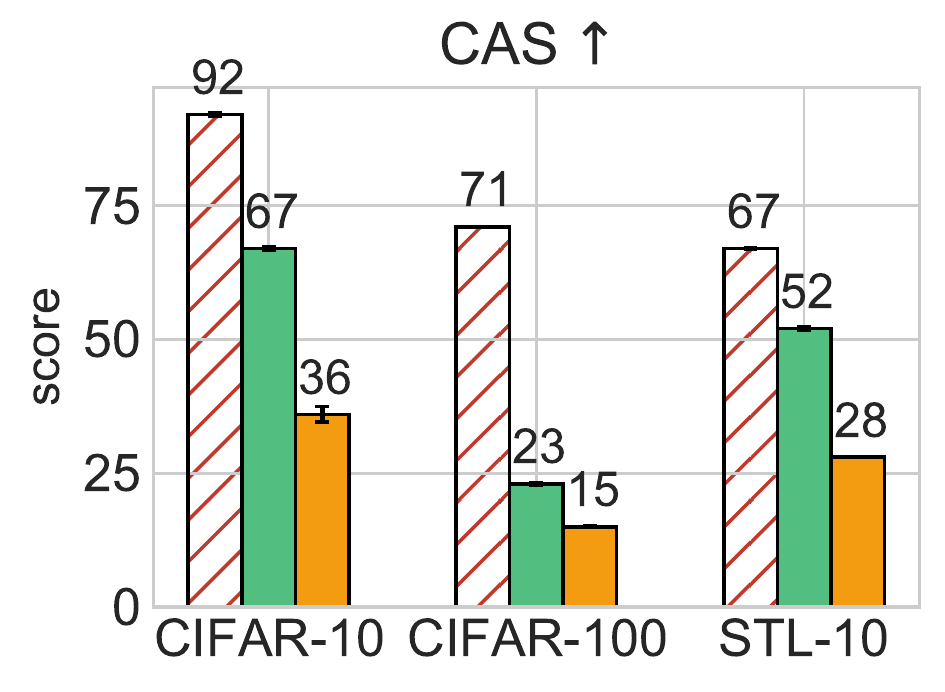}
	\caption{Comparison between GLO and COLA using the FID and CAS scores. Our model shows a clear advantage in all cases.}
	\label{fig:unsupervised}
	\vspace{-0.5cm}
\end{figure}

\mypar{Supervised.}
When learning from fully labeled datasets, we evaluate our model against the state-of-the-art conditional GAN variant CGAN-PD with varying sample sizes. Results are shown in Fig.~\ref{fig:full supervised}. 

Although conditional GANs obtain better FID scores on large datasets, their performance deteriorates rapidly when training size decreases. Furthermore, our model outperforms GANs when consulting the CAS score on almost all configurations. A qualitative comparison  presented in Fig.~\ref{fig:qualitative_gan_scola} and in Fig.~\ref{supp:fig:cgan_vs_scola_qualitative} in the Appendix suggests that this deterioration may be attributed to the mode-collapse manifested in CGAN when trained with insufficient data.
In contrast, in the extreme small sample regime the images synthesized by our model can hardly be distinguished from real images by both scores, suggesting superior generalization ability in this regime. 

\begin{figure}[hbt]
\captionsetup[subfigure]{labelformat=empty}
\centering
     \begin{subfigure}[]{0.5\textwidth}
    \hspace{-0.3cm}
     \begin{subfigure}[]{0.49\textwidth}
             \caption{\textbf{FID} $\downarrow$}

         \includegraphics[width=1\textwidth]{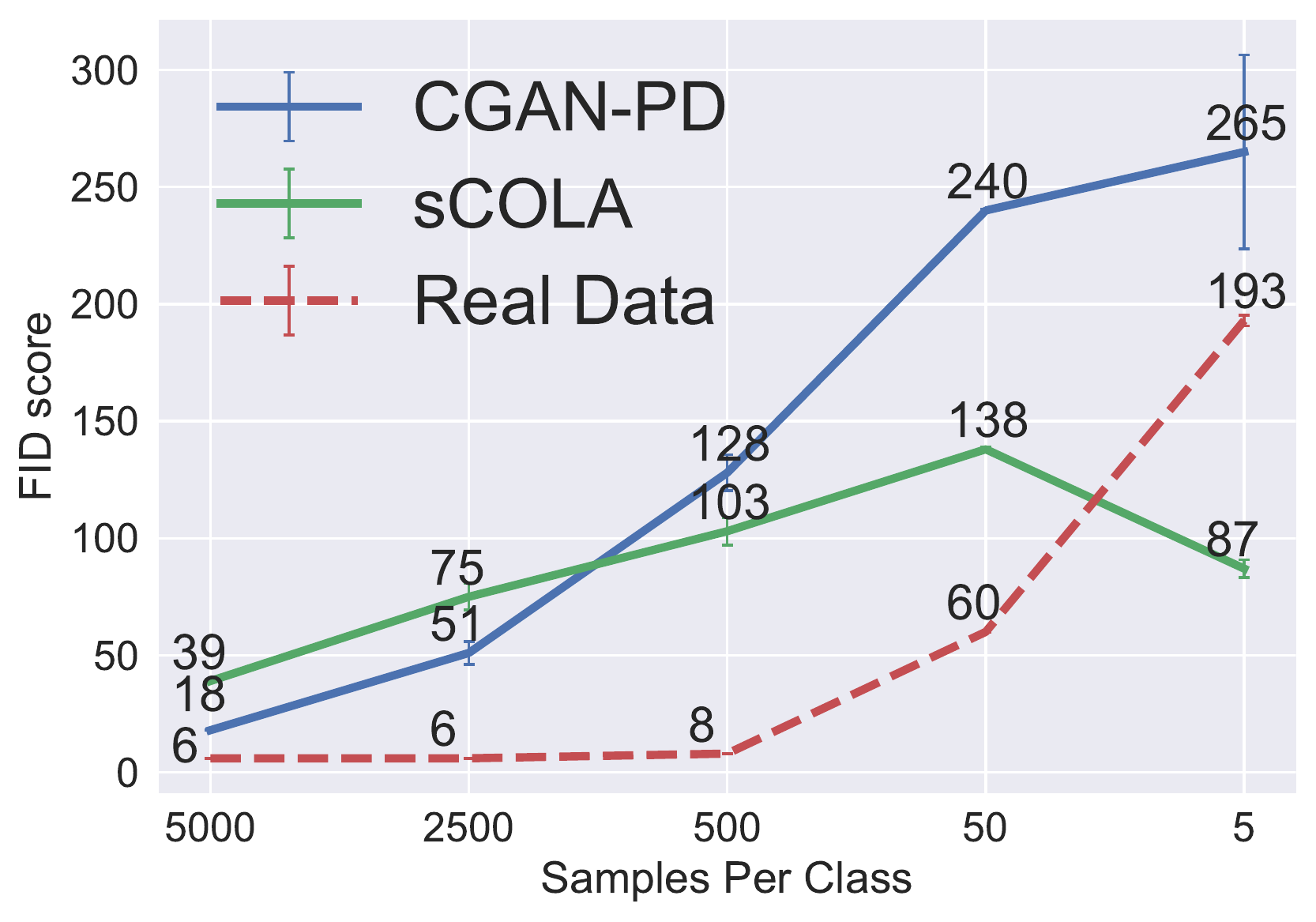}
     \end{subfigure}
     \begin{subfigure}[]{0.49\textwidth}
     \caption{\textbf{CAS} $\uparrow$}
         \includegraphics[width=1\textwidth]{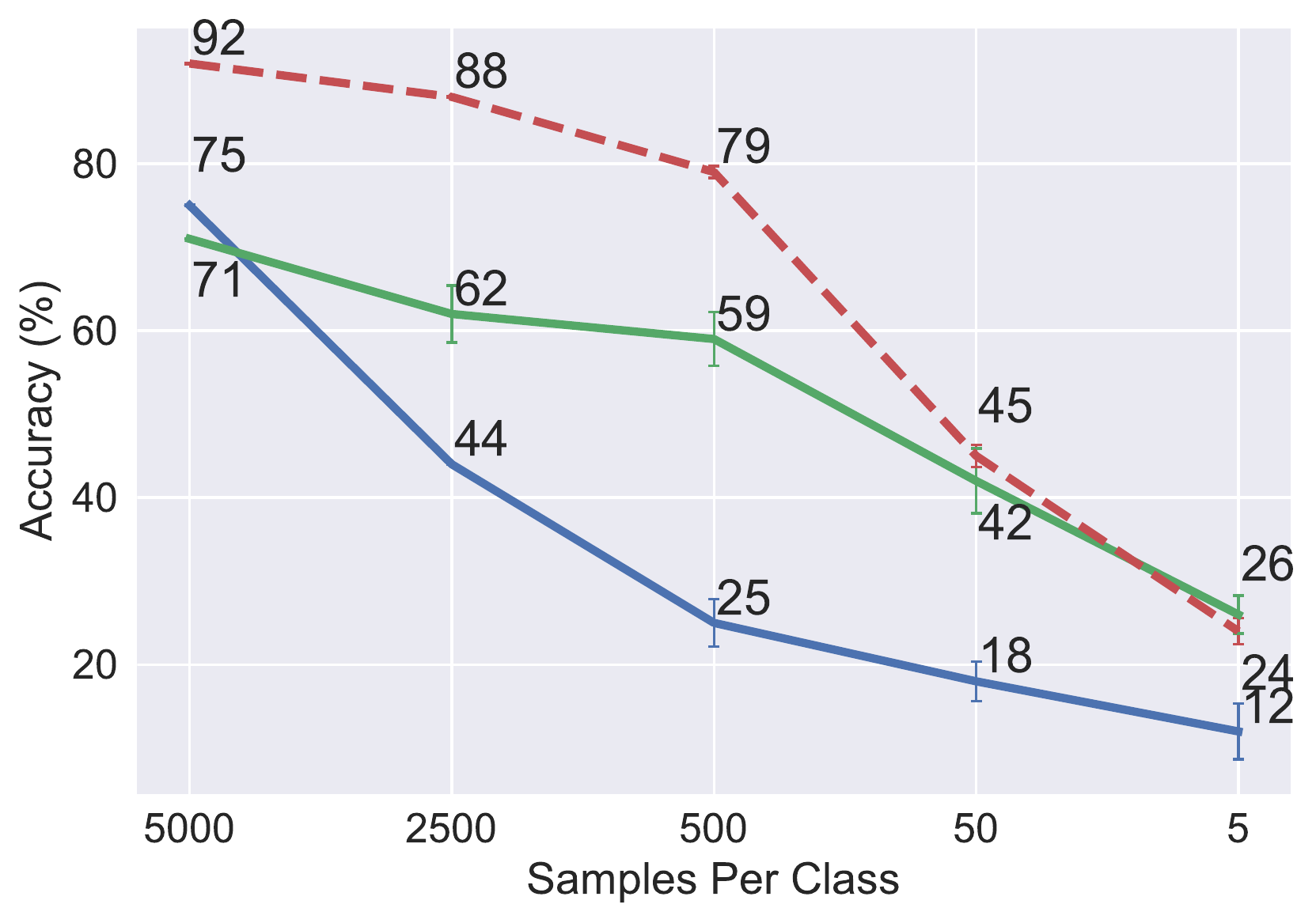}
     \end{subfigure}
     \vspace{-0.2cm}
    \caption{\scriptsize{CIFAR-10}}
    \vspace{0.3cm}
     \end{subfigure}
    \begin{subfigure}[]{0.5\textwidth}
    \hspace{-0.3cm}
     \begin{subfigure}[]{0.49\textwidth}
              \includegraphics[width=1\textwidth]{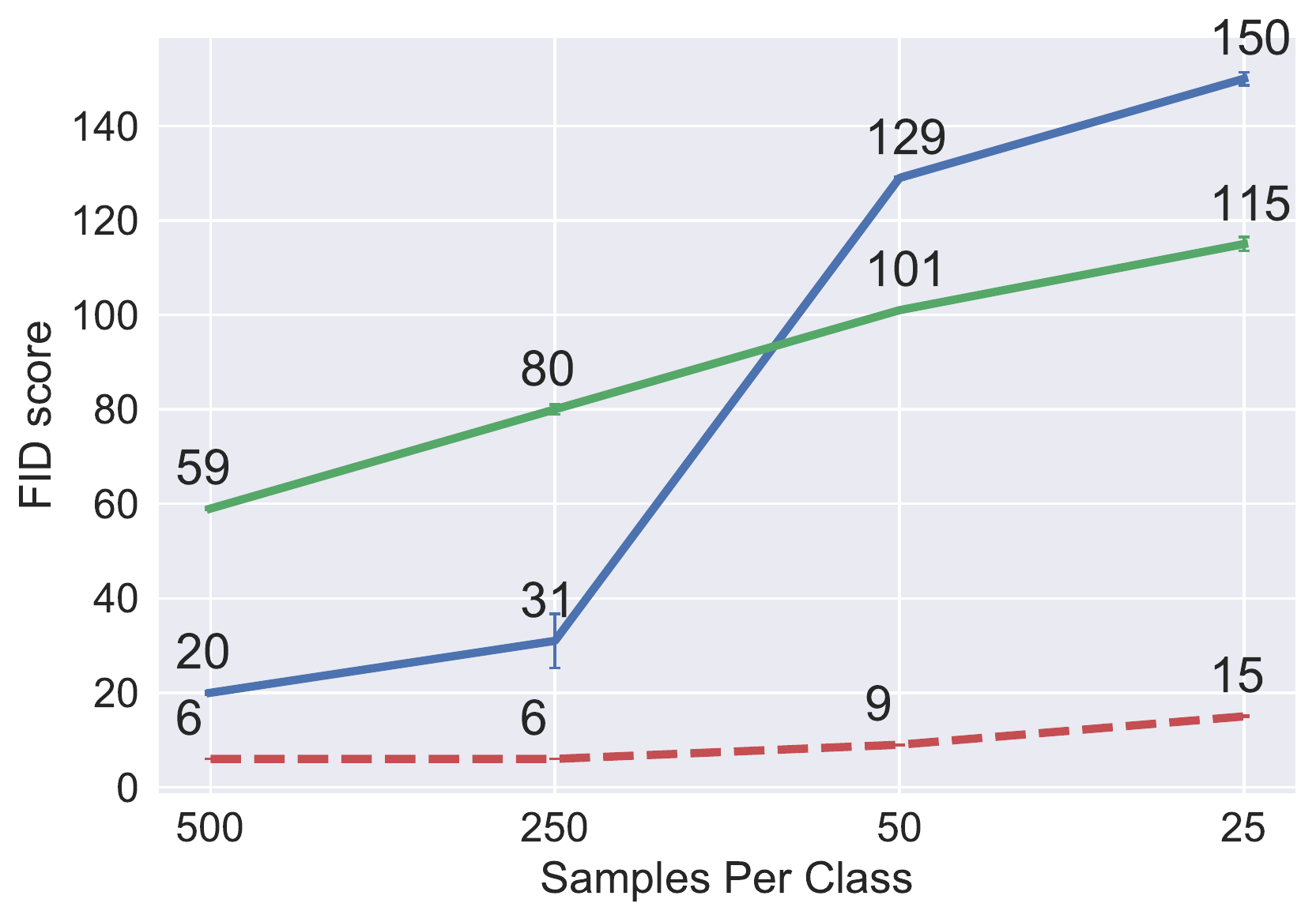}
     \end{subfigure}
     \begin{subfigure}[]{0.49\textwidth}
              \includegraphics[width=1\textwidth]{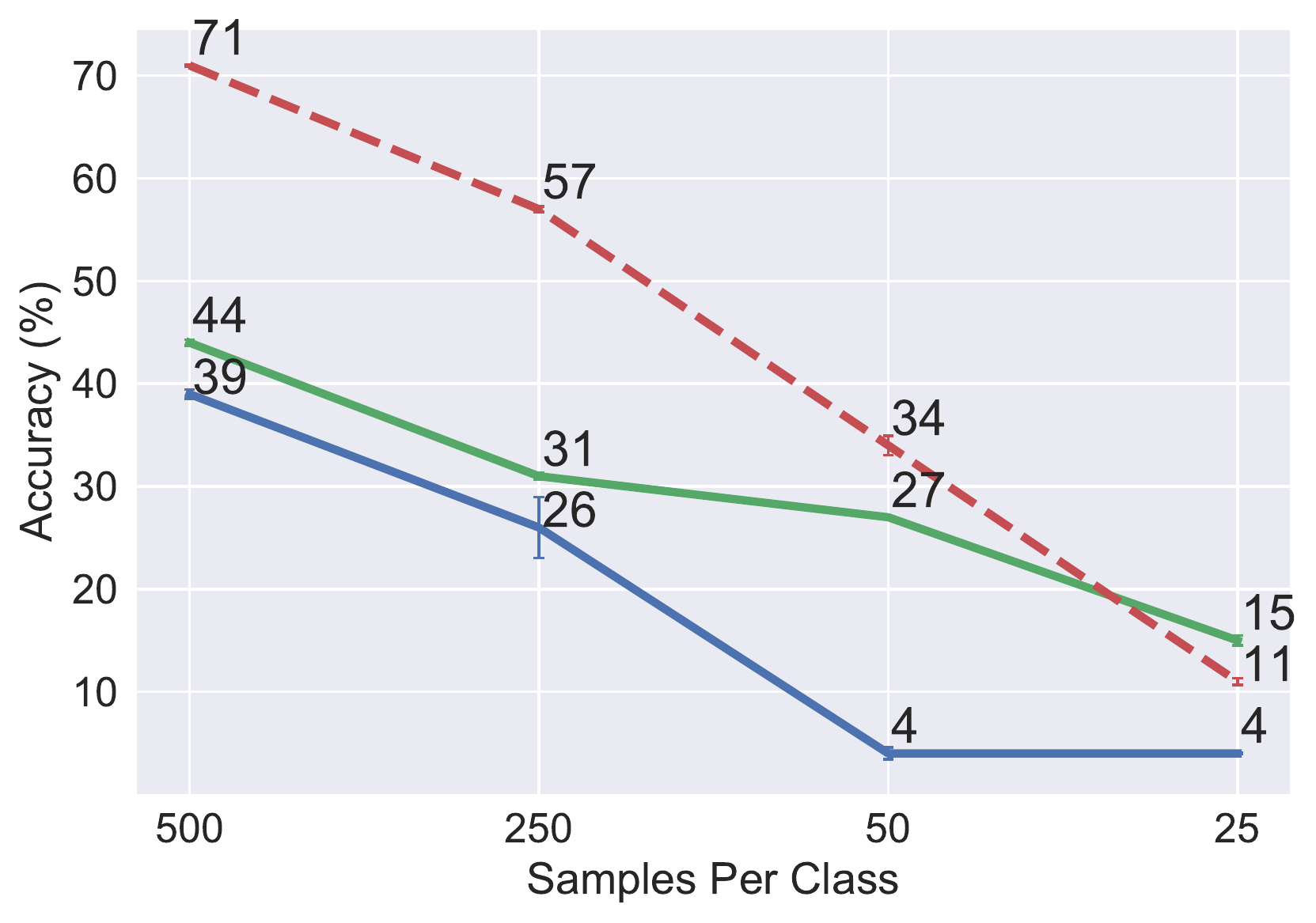}
     \end{subfigure}
     \vspace{-0.2cm}
    \caption{\scriptsize{CIFAR-100}}
    \vspace{0.3cm}
     \end{subfigure} 
          \begin{subfigure}[]{0.5\textwidth}
     \hspace{-0.3cm}
     \begin{subfigure}[]{0.49\textwidth}
              \includegraphics[width=1\textwidth]{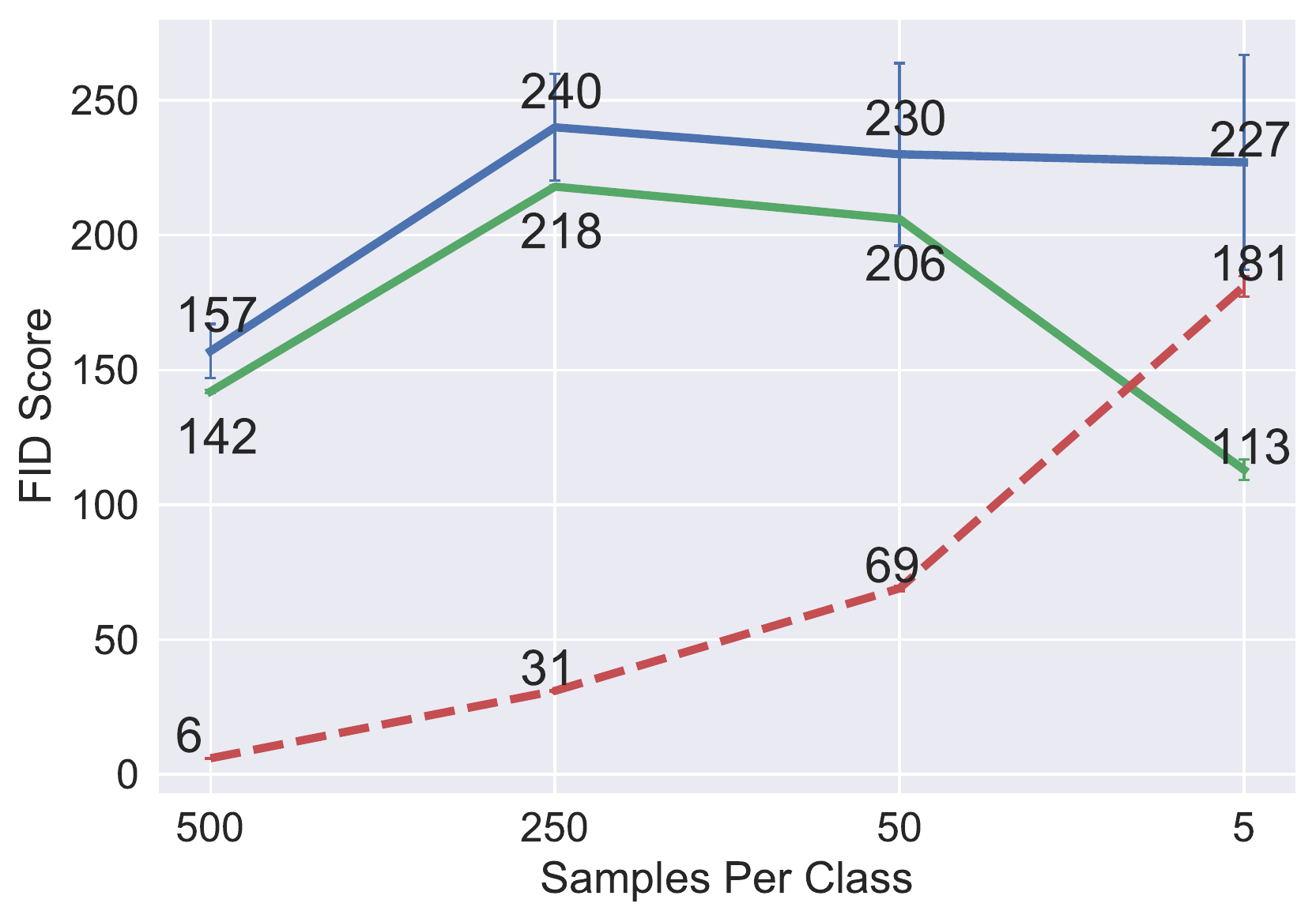}
     \end{subfigure}
     \begin{subfigure}[]{0.49\textwidth}
              \includegraphics[width=1\textwidth]{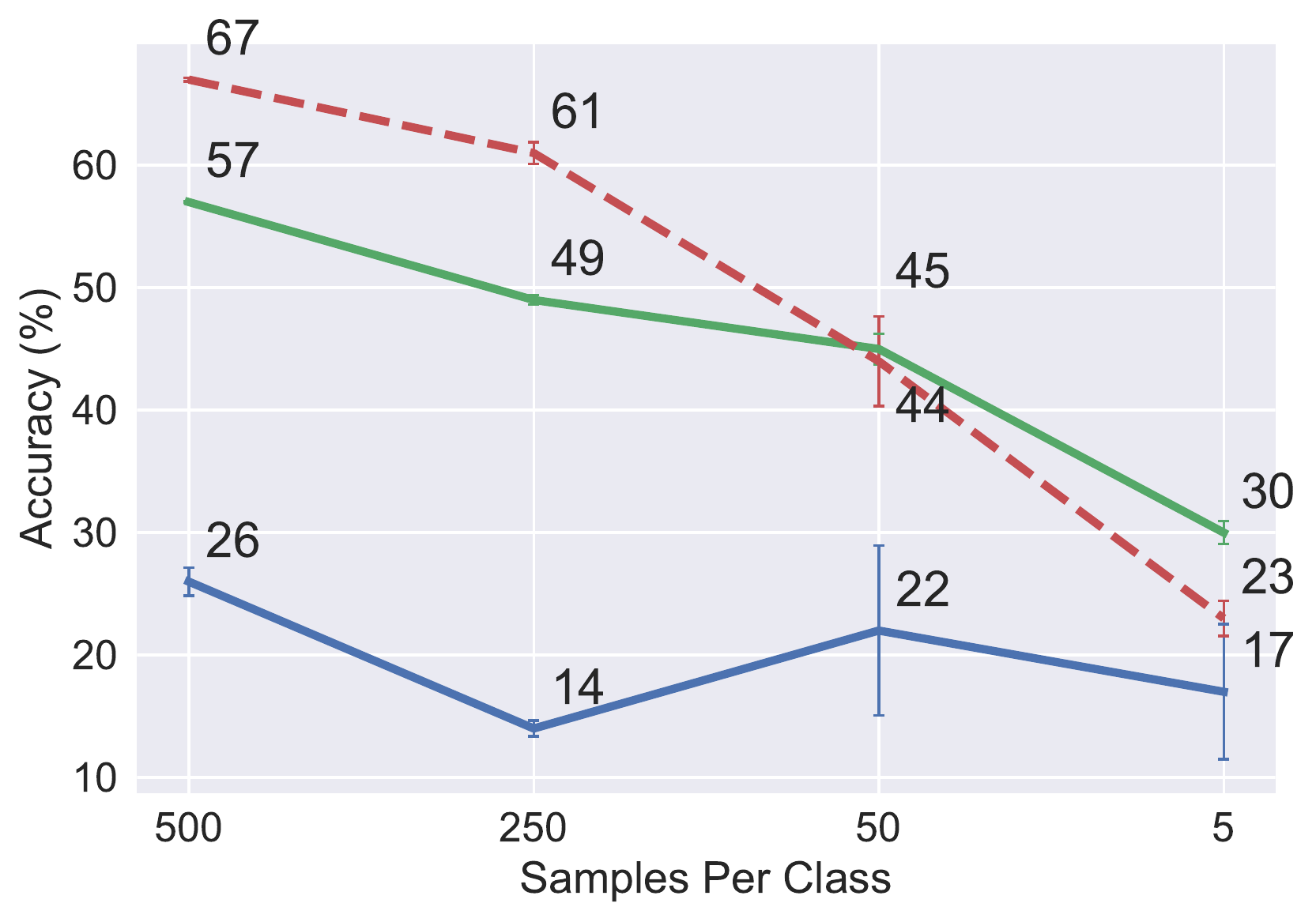}
     \end{subfigure}
     \vspace{-0.2cm}
    \caption{\scriptsize{STL-10}}
     \end{subfigure} 
        \caption{FID (left) and CAS (right) scores on CIFAR-10, CIFAR-100 and STL-10 with varying training sample sizes. sCOLA's generated images achieve better scores than the GAN's images in the small sample regime, and even achieve better scores than real images when data is extremely scarce (see also Section~\ref{sec:small-sample}).}         
        \label{fig:full supervised}
        \vspace{-0.5cm}
\end{figure}

\begin{figure}[ht]
\hspace{-0.5cm}
		\includegraphics{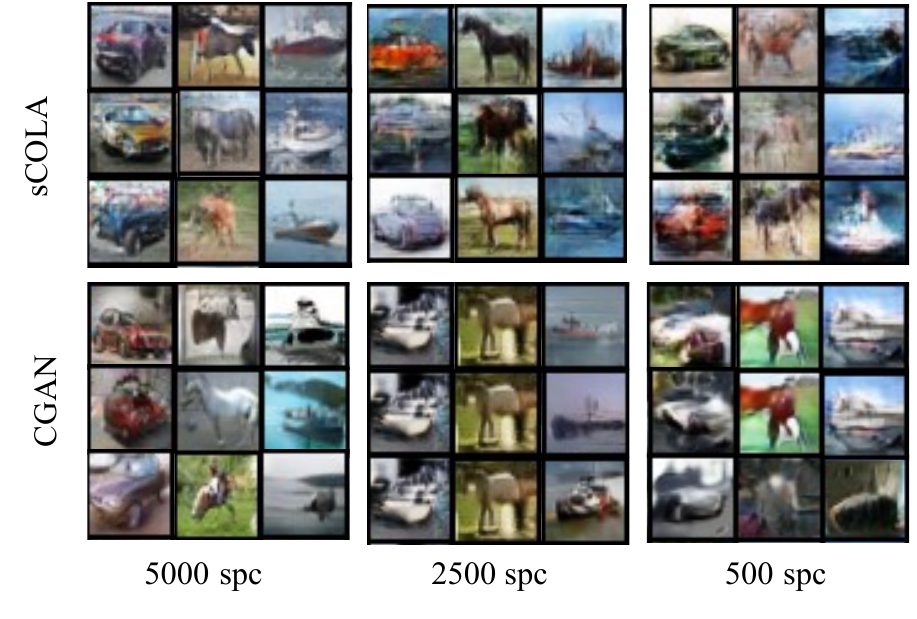}

\vspace{-0.5cm}
	\caption{qualitative comparison between sCOLA (top) and CGAN (bottom) trained on CIFAR-10 with varying numbers of samples per class (spc). Each column corresponds to a different class in the data. CGAN evidently suffers from mode-collapse when given insufficient data for training.}
	\label{fig:qualitative_gan_scola}
	\vspace{-0.5cm}
\end{figure}

% -----------------------------------------------------------------------

\section{Classification from Small Sample}
\label{sec:small-sample}

In this section we show the benefits of using our method in the small sample regime, where only a small sample is available to train the classifier, and \textbf{no external information can be used}. We will show that using our model to augment the small training set significantly improves the performance of a deep network classifier trained on this data. 

\mypar{Classification approach.}
sCOLA is first trained on the small training sample, and then used to generate novel samples from each class. The synthetic images are then combined with the real images, resulting in an extended training set (termed "Mix") that consists of 50\% real images, and 50\% synthetic images generated by our model. This extended set is then used to train a CNN classifier. For comparison, we train the same CNN classifier with the original images, making sure that both methods see the same subset of images with an identical training procedure.

\subsection{Methodology}

\begin{figure*}[ht]

	\centering
		\includegraphics[width=0.49\linewidth]{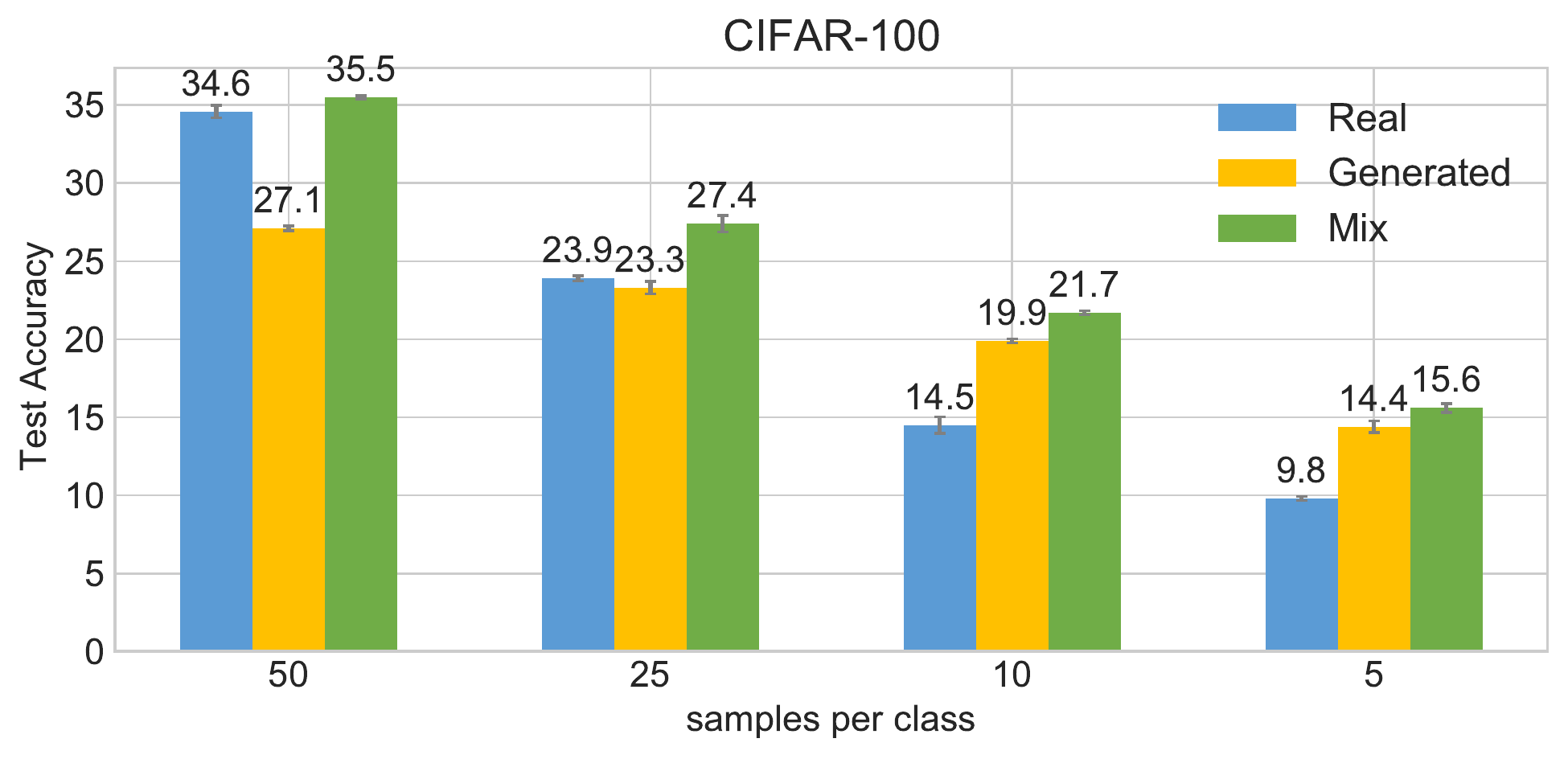}
		\includegraphics[width=0.49\linewidth]{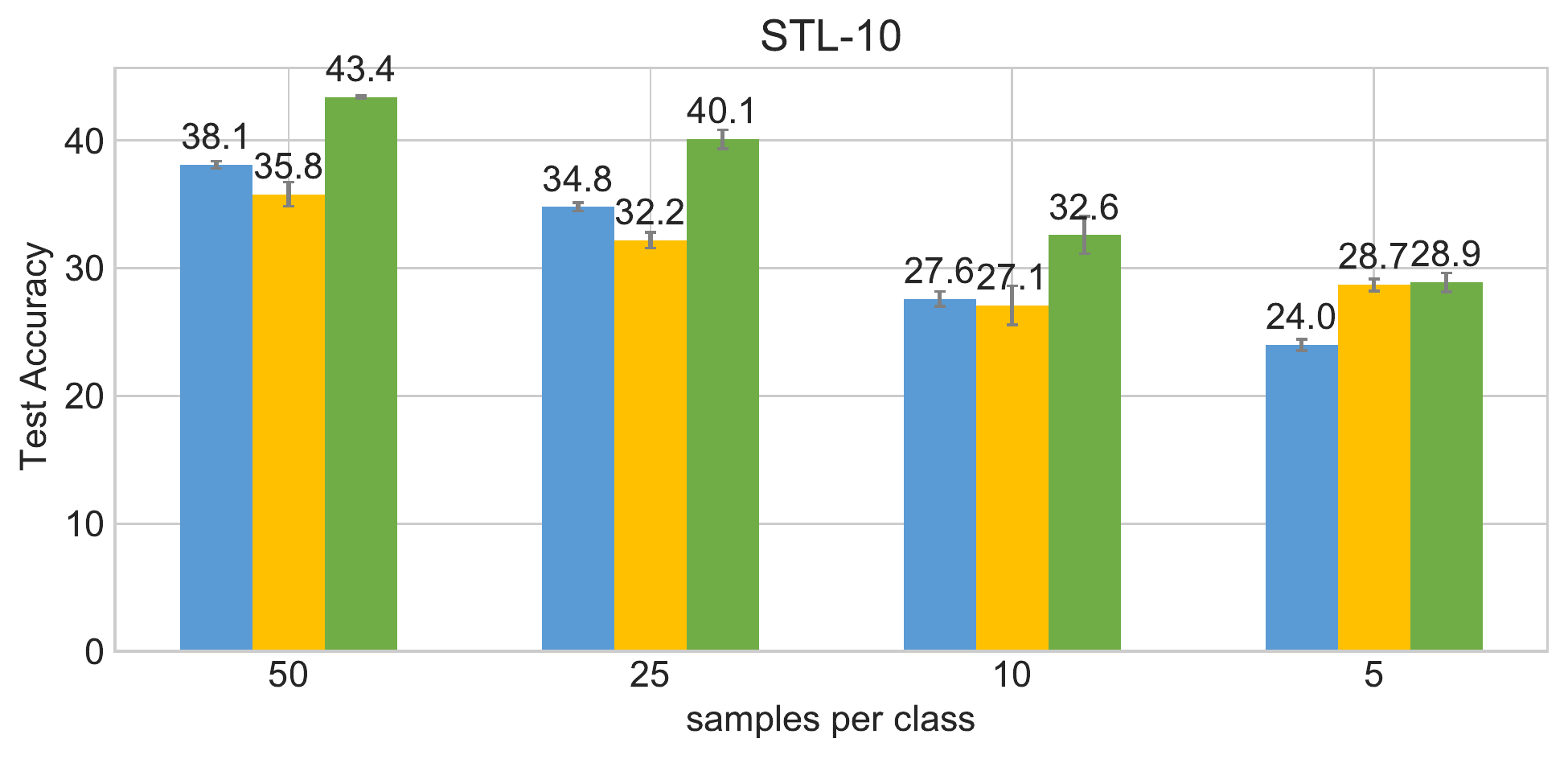}
	\caption{Classification accuracy for CIFAR-100 (left) and STL-10 (right) with varying sample size per class.}
	\label{fig:ss both}
\end{figure*}

\begin{table}[h]
\vspace{-0.5cm}
\footnotesize
\begin{center}
\begin{tabular}{l|ccc}
\thead{\bf Name} & \thead{\bf Classes} & \thead{\bf SPC \\ \bf Train/Test} & \thead{\bf Dimension}\\
										   
\hline
\rule{0pt}{3ex}
\hspace{-.5ex} \thead{CIFAR-10} \cite{krizhevsky2009learning}       & 10    & 5000 / 1000   & \small$32\times32\times3$\\
\thead{CIFAR-100}  \cite{krizhevsky2009learning}     & 100   & 500 / 100   & \small$32\times32\times3$\\

\thead{STL-10 \footnotemark   \\ \tiny{(downsampled, labeled only)} } \cite{coates2011analysis}       & 10   & 500 / 800   & \small$48\times48\times3$\\

\thead{Tiny ImageNet} \cite{le2015tiny}      & 200   & 500 / 50   & \small$64\times64\times3$\\
\end{tabular}
\end{center}
\vspace{-0.5cm}
\caption{Datasets used in our experiments.}
\label{table:datasets}
\vspace{-0.15cm}
\end{table}

\footnotetext{Unlike most generative models trained on STL-10, in this work we only use the labeled images, and discard the 100K unlabeled images.}

The datasets we use are described in Table~\ref{table:datasets}. For each dataset we train our method with various sample sizes, ranging from 100 samples per class (spc) to as low as 5 spc. For each sample size we run our model on 3 random samples of the same size, and evaluate the classifier's accuracy on the original held-out test set of the data. 
In order to isolate the contribution of our approach from other factors, we fix the classifier's architecture to an off-the-shelf ResNet-20 \cite{he2016deep} for all datasets except for Tiny ImageNet, which, due to its larger size, resolution and number of classes, necessitates the use of a larger network. Consequently we use the same WRN-16-8 \cite{zagoruyko2016wide} network as CFVAE-DHN (excluding DHN initialization). Full implementation details can be found in Appendix~\ref{supp: ss classification implementation}.

\subsection{Empirical Results}

We compare our model trained on CIFAR-10 and Tiny ImageNet with the best published results reported in \cite{lin2020efficient}. A short description of these methods can be found in Section~\ref{ss related work}. The results are summarized in Table~\ref{table:ss benchmark}. Our method achieves the best results across all sample sizes on both datasets.

\begin{table}[ht]

\begin{center}

\footnotesize{\textbf{CIFAR-10}\\}\vspace{0.1cm}
\begin{tabular}{ l| c c c c c}

\toprule

& 100& 50 & 20 & 10 &5  \\
 \midrule 
 
 DADA& \ttc{48.32}{.23}& \ttc{40.48}{.57}    & \ttc{30.44}{.37}  &  \ttc{21.67}{.58} & -\\
 
Tanda & \ttc{45.17}{1.84}  & \ttc{39.16}{1.18}    &  \ttc{29.84}{1.23}  & \ttc{20.18}{.73} & -  \\

CFVAE-DHN  &  \ttc{55.58}{.12}  &    \ttc{52.06}{.36}&  \ttc{32.65}{.38} &  \ttc{34.11}{.67}  &     - \\

\textbf{sCOLA}  &  \ttc{\textbf{ 58.59}}{0.58}  & \ttc{\textbf{54.51}}{0.22} & \ttc{\textbf{49.63}}{1.29} & \ttc{\textbf{42.86}}{2.04} & \ttc{\textbf{ 29.05}}{1.09} \\

\vspace{0.3cm}
\end{tabular}

\textbf{Tiny-ImageNet\\}\vspace{0.1cm} \hspace{-1.2cm}
\begin{tabular}{ l|cccc }

 \toprule

 & 100 & 50  & 20    & 10  \\
 \midrule 
 
DADA & \ttc{17.64}{.82} & \ttc{14.97}{1.08} & \ttc{10.13}{2.04}  & -    \\

Tanda  & \ttc{27.07}{.94} & \ttc{17.95}{.59} &  \ttc{13.92}{.59}  & -    \\
 
CFVAE-DHN & \ttc{\textbf{35.97}}{.35} & \ttc{28.82}{.79}  & \ttc{21.37}{.29}  &  - \\
 
\textbf{sCOLA} & \ttc{\textbf{35.24}}{.34} &  \ttc{\textbf{29.70}}{.05}  & \ttc{\textbf{23.99}}{.52}  & \ttc{\textbf{17.14}}{.27}   \\

\end{tabular}

\end{center}
\vspace{-0.5cm}
\caption{Classification accuracy for \textbf{CIFAR-10} (top) and \textbf{Tiny-ImageNet} (bottom). Each column corresponds to a different sample size per class. The architecture used by our method is smaller or similar to the ones used by the other methods (see methodology).}
\label{table:ss benchmark}
\vspace{-0.5cm}
\end{table}

Additionally, we expand our experiments to datasets with no published results to date on small sample classification tasks. For these datasets, we show that when a classifier is trained on a mixed dataset consisting of real and synthetic images, it yields better results compared to those obtained when being trained only on the real images or only on the synthetic images. This suggests that our model succeeds in learning the data distribution well enough, and can subsequently generate novel samples that do not exist in the real data. Fig.~\ref{fig:ss both} shows results on CIFAR-100 and STL-10.

% -----------------------------------------------------------------------
\subsection{Ablation Study}
An ablation study probing different components of our method, including image similarity measure,  architecture of the generator and classification network, can be found in Section~\ref{supp: ablation_sec} in Supp. In this regard, we found that the size of the classifier has little effect on the overall performance when dealing with very small samples.

% -----------------------------------------------------------------------
\section{Summary and Discussion}

We described a novel unsupervised non-adversarial generative model that is capable of generating diverse multi-class images. This model outperforms previous non-adversarial generative methods, and outperforms more complicated GAN models when the training sample is small. 

Unlike GAN models, our model is characterized by stable and relatively fast training, it is relatively insensitive to the choice of hyper-parameters, and it has control over each class variance in the synthesized dataset.  
Furthermore, empirical results show that our method is robust to the risk of mode-collapse, which plagues most GAN models when trained with insufficient data.

We further demonstrated the capability of our model to augment small data for classification, advancing the state-of-the-art in this domain. Since our method is only used to augment the small sample, it remains orthogonal to future advances in algorithms devised for small training sets.

% -----------------------------------------------------------------------

\section*{Acknowledgements}
This work was supported in part by a grant from the Israel Science Foundation (ISF) and by the Gatsby Charitable Foundations.

\nocite{arjovsky2017wasserstein, kodali2017convergence, berthelot2017began}

{\small
\bibliographystyle{ieee_fullname}
\bibliography{main}
}

%------------------------------------------------------------------------------
\setcounter{section}{0}
\renewcommand\thesection{\Alph{section}}

\title{Appendix}
\author{}
\date{}
\maketitle
\vspace{-2cm}\section{The FID score is inadequate for multi-class datasets}\label{supp: fid inadequacy}

In this section we will show that the FID fails to reveal intra-class variance, highlighting its inadequacy to serve as a single metric for assessing generative models on multi-class data. To do so, we will use our model to construct two datasets that obtain similar FID scores, but exhibits an apparent difference in terms of the intra-class variance. 
We notice that generating images from latent codes that reside in proximity yields images that are visually and semantically similar. On the other hand,  sampling  latent codes that come from the same latent cluster but with a greater distance from each other,  yields far more diverse outputs under the model. An example of this effect is illustrated in Fig.~\ref{supp: sampling}

\begin{figure*}[ht]
    \begin{center}

    \includegraphics[]{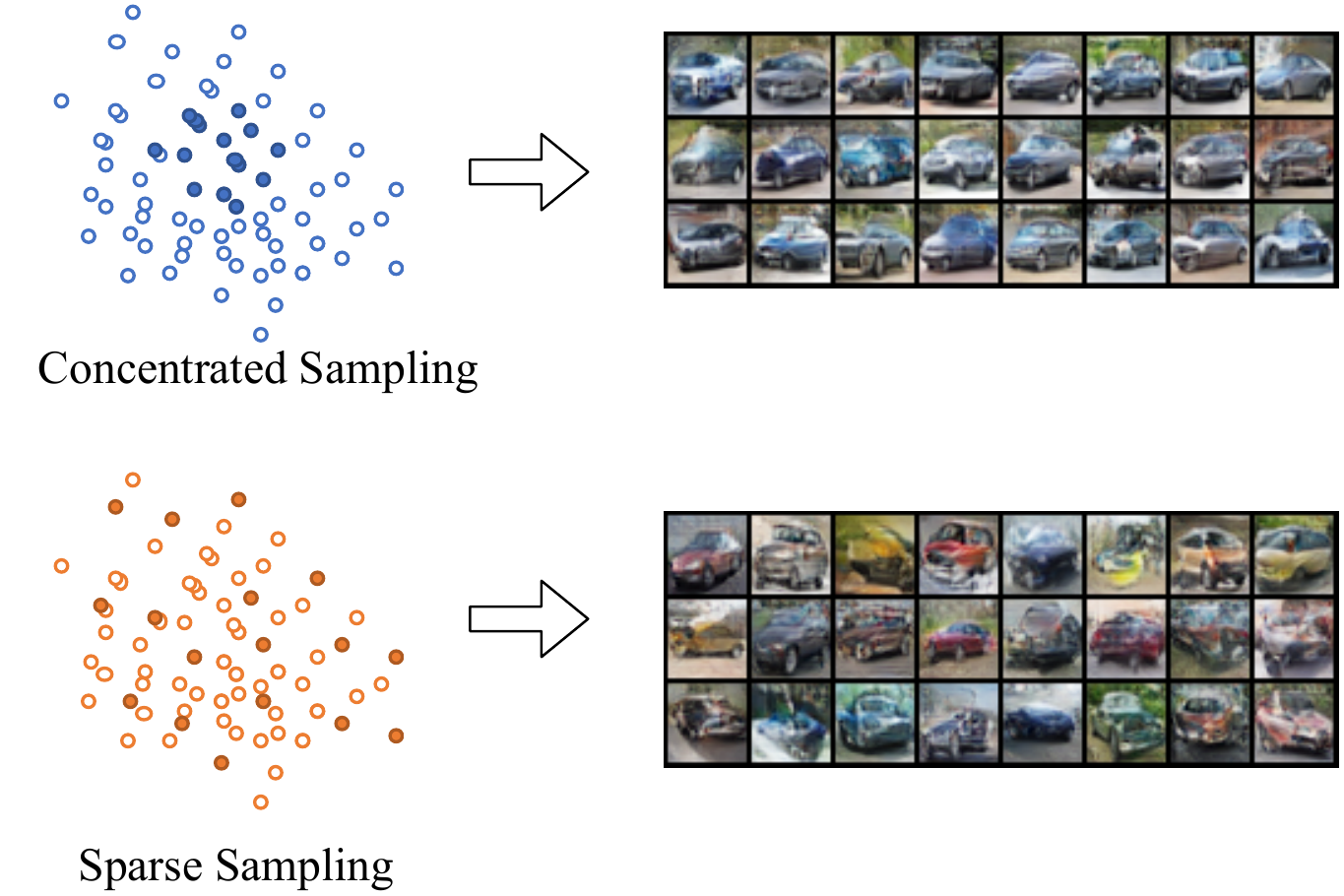}
    \end{center}
      \caption{The effect of the scattering of latent codes on the generated images. in the top row, latent codes that are sampled around the cluster means result in similar images with small intra-class variation. In the bottom row, latent codes that are sparsely sampled result in images that exhibit a greater intra-class variance. }
      \label{supp: sampling}
    
\end{figure*}

Consequently, we can generate two versions of synthesized datasets- one where each class is generated from concentrated latent codes, and one where each class is generated from sparsely sampled latent code from the same cluster. The first dataset will consist of homogeneous classes, exhibiting a small intra-class variance whereas the second one will hold classes with a larger variety of objects.

In this experimental setting, we use our model to generate two synthetic versions of CIFAR-10- one which we will term 'concentrated' which is made of generations sampled from latent codes concentrated around the cluster means, and a second dataset, termed 'sparse' which is based on sparsely sampled latent codes from the same cluster.
We then evaluate these synthetic datasets using the FID and CAS scores.results are presnted in Fig.~\ref{supp: fid cas discrepancy}

Since the 'concentrated' dataset has low intra-class variation, each class consists of similar images, which makes it an inferior dataset for training a classifier (as is evident by the low accuracy obtained by a classifier trained on this data). On the other hand, the 'sparse' dataset is characterised by a higher intra-class variance, with diverse images in each class, yielding an effective training set for classification. Nevertheless, the FID scores of the two datasets are barely affected by these differences, since it is based on a single multivariate Gaussian approximation of their activations in the penultimate layer of the Inception network, which cannot capture intra-class variance.

        \begin{figure}[ht]
        \begin{center}
        \includegraphics[scale=0.5]{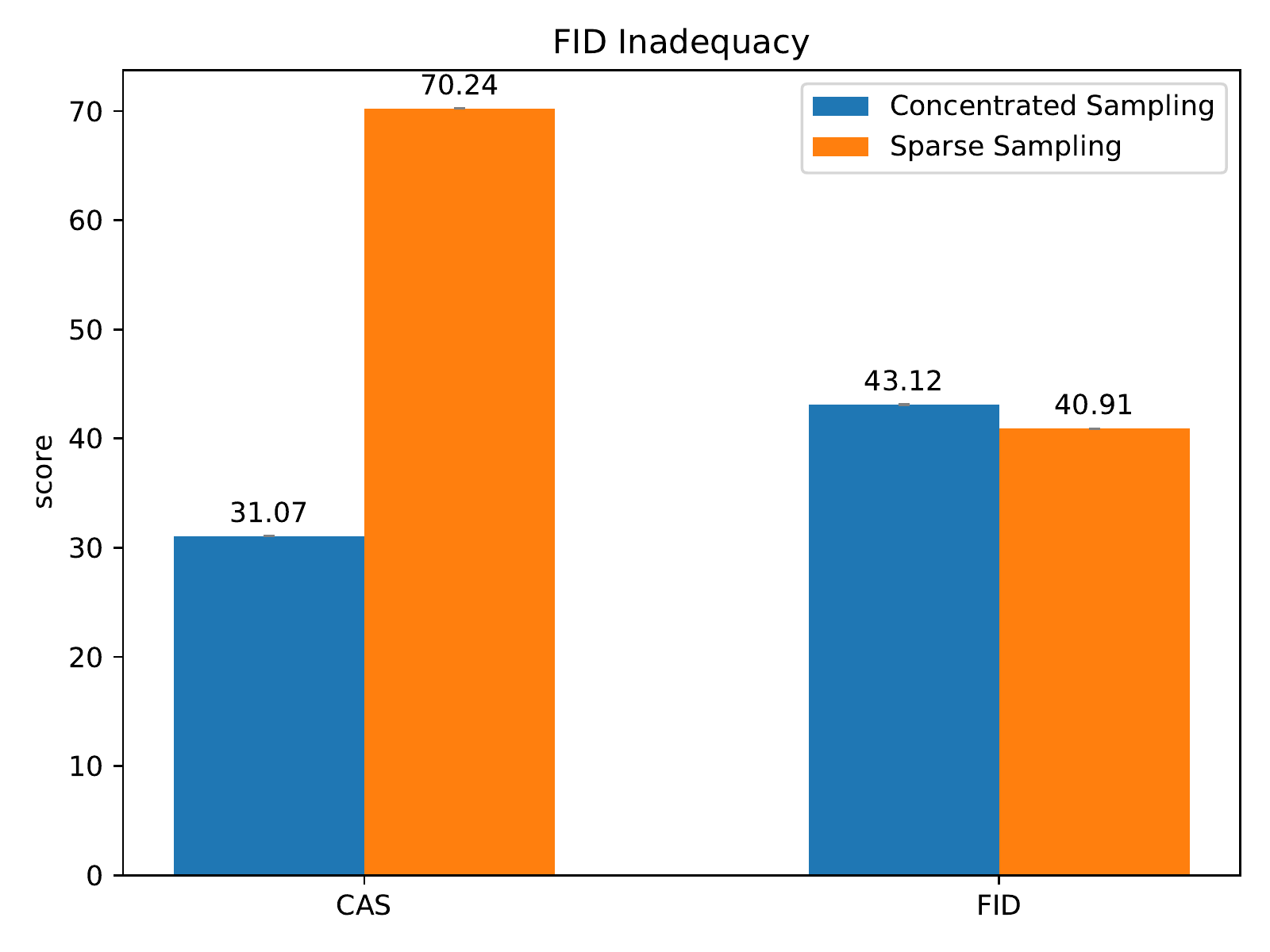}
        \end{center}
        \caption{While the CAS score is an informative measure of the intra-class variance, the FID fails to discriminate between the two datasets. }
        \label{supp: fid cas discrepancy}
        \end{figure}

% -----------------------------------------------------------------------

% \section{Algorithm Description}

    \begin{algorithm}[ht] 
            \caption{Clustering the Latent Space} 
            \begin{flushleft}
            
                \textbf{INPUT:} Unlabeled dataset $\{x_i\}_{i=1}^N$\\
                 \hspace{33pt} $K$ - number of clusters\\
                \hspace{33pt} $E_\theta$ - dual-head CNN Encoder with parameters $\theta$ \\
                \hspace{33pt} $E_\theta^1(x) \in \mathbb{R}^K ;  E_\theta^2(x) \in \mathbb{R}^4$ \\
                \hspace{33pt} $\{\xi : X \rightarrow X\} $ random augmentation functions\\
               
                \hspace{33pt} $\sigma$ - intra-class variance\\
                \hspace{33pt} $\lambda_e$ - learning rate at epoch e\\
                \textbf{OUTPUT:} a matching  $\{(x_i, t_i)\}_{i=1}^N$    such that  $\{t_i\}_{i=1}^N$ are clustered according to their matching   $x_i$
            \end{flushleft}
            \begin{algorithmic}

            \For {$i$=1 to $N$}
            \Comment{initialize a random matching}
            \State sample $c \sim Categ(\frac{1}{K}, ... \frac{1}{K})$ 
            \State sample $v \sim \mathcal{N}(\vec{\mu_c}, \sigma I_{K\times K})$
            \State $t_i \leftarrow \frac{v}{||v||_2}$ , $A  \leftarrow \{(x_i, t_i)\}_{i=1}^N$
            \EndFor

            \For {\textit{e=1...epochs}}
            \For {\textit{i=1...iters}}
            
            \State sample batch $(X^b, Z^b)$
             
            \State solve the assignment problem for $\{(x_i, t_i)\}_{i=1}^b$ 
            \State using the Hungarian Algorithm

            \State $$ \pi^* \leftarrow  \argmin_{\pi : [b] \leftrightarrow [b]} \sum_{j} ||E_\theta^1(x_j) - t_{\pi(j)}||_2^2  $$
            
            \State  $$ \theta \leftarrow \theta - \lambda_e \nabla_{\theta} \sum_{j} \sum_{\xi}||E_\theta^1(\xi(x_j)) - t_{\pi^*(j)}||_2^2 $$
             
            \State  $\forall i \in b \quad A(x_i) = t_{\pi^*(i)}$
            \Comment{update assignments}

            \EndFor
            \For {\textit{i=1...iters}}
            \State sample batch $(X^b)$
            \Comment{rotate batch}
            \State $X_{R}^b \leftarrow \big\{(X^b)_r | r \in  \{0^{\circ}, 90^{\circ}, 180^{\circ}, 270^{\circ}\} \big\}$
            
            $$ \theta \leftarrow \theta - \lambda_e \nabla_{\theta} \frac{1}{|X^b_R|}\sum_{i \in b} l_{ce}(E_\theta^2(x_i), r(x_i)) $$
            \Comment{$l_{ce}$ is the cross-entropy loss}
            
            \EndFor
            \EndFor

        \end{algorithmic}
        \label{supp: clustering pseudocode}
    \end{algorithm}
    
%------------------------------------------------------------------------------

 \section{Implementation details} \label{supp: implementation_sec}
 
\subsection{\textbf{Step I - Clustering the latent space.}}\label{supp: step1 implementation}
For all experiments we use a ResNet-18 \cite{he2016deep}  network for the encoder.
The network is trained with SGD with an initial learning rate of 0.05 and momentum of 0.9 for 200 epochs. Learning rate is decayd by a factor of 0.5 every 50 epochs.
Training is done sequentially where an epoch optimizing the target assignment problem is followed by an epoch optimizing the rotation prediction problem. In both cases we use a batch size of 128, where in  the target assignment problem images are augmented by cropping, flipping and color jitters, and in the rotation prediction task each image is rotated in all orientations, yielding a batch size of 512. 
Weight decay regularization of 0.0005 is used on all datasets.

\subsection{\textbf{Step II - Image generation.}}\label{supp: step2 implementation}
In all our experiments we used the ADAM optimizer \cite{kingma2014adam}, with an initial learning rate of 0.01 for the latent code, and 0.001 for $G_\theta$.
The generative model was trained for 500 epochs, learning rate was decayed by 0.5 every 50 epochs. The only parameters that change throughout our experiments are the choice of architecture for the generator, the choice of reconstruction loss and the dimensionality of the latent space as follows: 

\begin{enumerate}
    \item \textbf{Small-Sample}: In this section, the generative function $G_\theta$ shares the same CNN generator architecture used in InfoGAN \cite{chen2016infogan} (which is also the architecture used in GLO \cite{bojanowski2017optimizing}). The dimension of the latent space was set to $\mathcal{Z} \subset \mathbb{R}^{K+d}$  where K is the number of classes in the dataset, and $d=64$ for all datasets. 
    $\mathcal{L}_{rec}$ is implemented using the unsupervised Laplacian-Pyramid loss Eq.~\ref{eqn:lap_loss}. An ablation study of different similarity measures is presneted in Section~\ref{supp: perceptual ablation}.
    
    \item \textbf{Full Data}: In the supervised version  (sCOLA), the generative function $G_\theta$ shares the same CNN generator architecture used in CGAN \cite{miyato2018cgans}, while in the unsupervised framework (COLA) we use the generator of InfoGAN \cite{chen2016infogan}. In both versions $d=128$,  and $\mathcal{L}_{rec}$ is implemented using the perceptual loss Eq.~\ref{eqn:perceptual_loss}. 
\end{enumerate}

\subsection{\textbf{Small-sample classification}}\label{supp: ss classification implementation} For a fair comparison, we use the same training procedure on all data sizes. i.e. same batch size, number of epochs, iterations per epoch and learning schedule. 
   
ResNet-20 was trained for 180 epochs, with an initial learning rate of 0.1,  decayed by 0.5 every 30 epochs.
Whereas WRN-16-8 was trained for 200 epochs with an inital learning rate of 0.1, decayed by 0.2 every 60 epochs. Both networks were trained using SGD optimization with a batch size of 128.  An ablation study showing the effect of different architectural designs, is presented in Section~\ref{supp: classifier_arch}.
The classifier was trained using standard data augmentation with such image transformations as random flip and crop.

\subsection{\textbf{FID score implementation}} \label{supp: fid implementation}
For CIFAR-10 and CIFAR-100, FID scores were computed on a sample of 10K generated images against the default Test-set of size 10K. Each model was trained 3 times, and the final score was taken as the average over 10 random samples from each model. 

For STL-10, FID scores were computed on a sample of 8K generated images against the default Test-set of size 8K. Note that most generative models that had been evaluated on this dataset used the whole dataset, including the 100K unlabeled images. Therefore these previous results are not comparable to the experimental results reported here.

 \section{Ablation Study} \label{supp: ablation_sec}
 \subsection{Image similarity measure} \label{supp: perceptual ablation}
 Many generative models are trained using some form of reconstruction loss that is based on a similarity measure between the original image and the one reconstructed by the model. Since using the Euclidean distance in pixel space is highly inadequate  (similar objects may differ drastically in pixel space) alternative similarity measures that capture the perceptual relationship between images have been thoroughly investigated in recent years. In this context, it has become a common practice to use a perceptual loss based on a VGG network that was pre-trained on ImageNet for image synthesis tasks \cite{dosovitskiy2016generating, johnson2016perceptual, hoshen2019non}. Later works \cite{zhang2018unreasonable} have evaluated numerous similarity measures that are based on deep features of neural nets and concluded that they exhibit a strong correlation with human judgment even when these features where obtained in an unsupervised or self-supervised manner. Nevertheless, establishing a similarity measure without large amounts of data remains to this date an uncharted territory. In order to make our model applicable in the small-data regime, we seek a meaningful similarity measure that uses little to no data. Furthermore, while most works in this area have evaluated the similarity measure according to human judgment on image quality, we sought a measure that will also prove useful in downstream tasks. To this end, we have investigated various methods and evaluated them on the downstream task of image classification of the generated images. 
 Results on partitions of size 1\% of CIFAR-10 can be found in Fig.~\ref{supp:perceptual loss}. We found that the Laplacian Pyramid Loss used in \cite{bojanowski2017optimizing} yields the best results compared to other unsupervised measures.
 The loss functions that were compared are as follows:
 \begin{enumerate}
     \item \textbf{ImageNet VGG16} - The original perceptual loss described in Eq.~\ref{eqn:perceptual_loss} 
     \item \textbf{Laplacian Pyramid} - The Laplacian Pyramid Loss as described in Eq.~\ref{eqn:lap_loss}
     \item \textbf{k-means ResNet32} - Perceptual distance based on layers of ResNet32 initialized with stacked k-means as described in \cite{krahenbuhl2015data}
     \item \textbf{Random ResNet32} - Perceptual distance based on layers of ResNet32 initialized with random weights
     \item \textbf{CIFAR-10 ResNet32} - Perceptual distance based on layers of ResNet32 trained on CIFAR-10
      \item \textbf{L1} - The $\mathcal{L}_1$ loss in pixel space

 \end{enumerate}

\begin{figure}[ht]
	\centering
	\includegraphics[width=\linewidth]{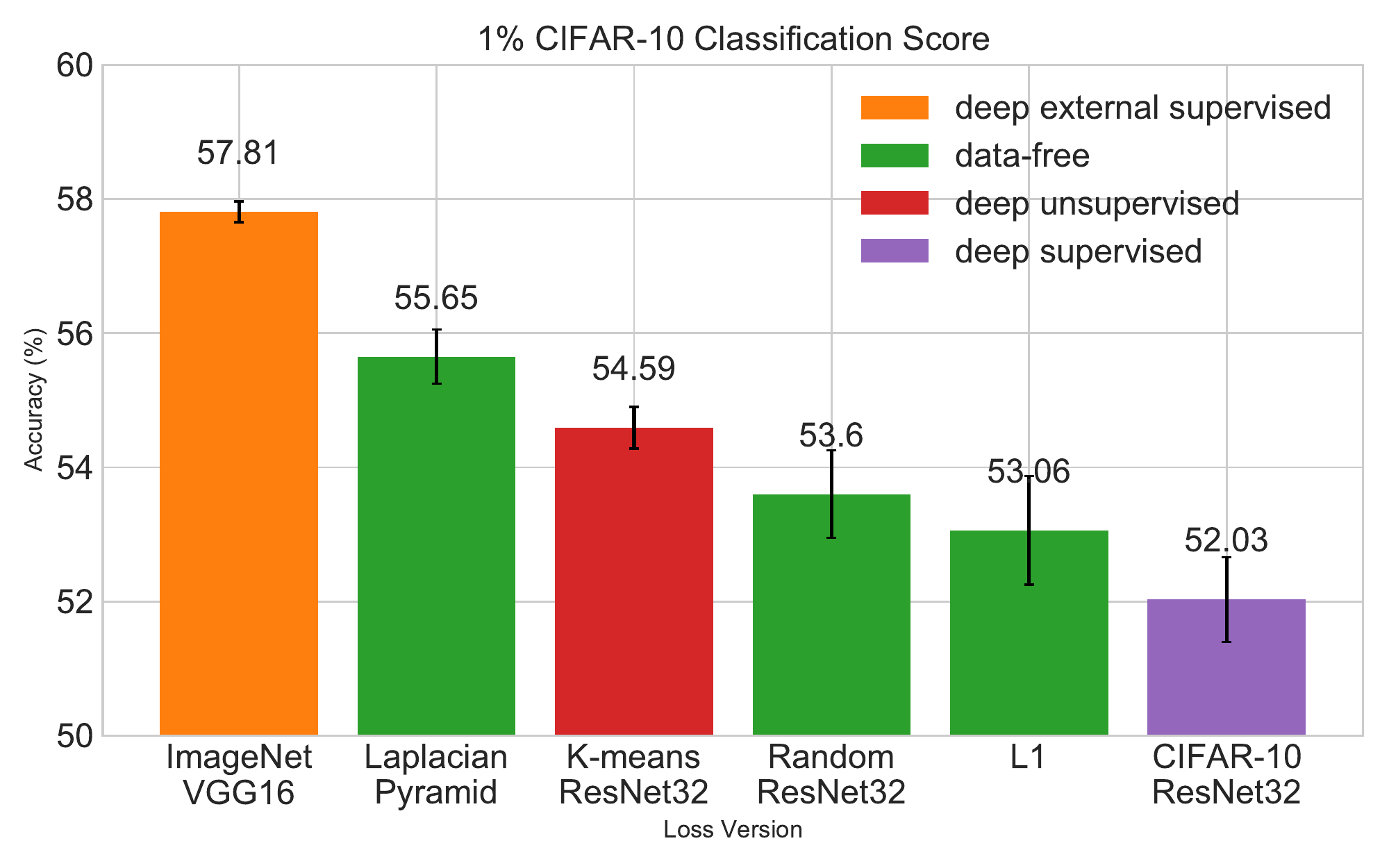}
	\caption{The effect of different similarity measures in the optimization of sCOLA when trained on on 1\% of the images in CIFAR-10. Classification score is obtained using the same framework described in Section~\ref{supp: ss classification implementation}}
	\label{supp:perceptual loss}
\end{figure}

\subsection{Small Data Classifier Architecture.} \label{supp: classifier_arch}
We experimented with various architectures to check whether larger networks are preferable over smaller ones in the small-sample regime. Our results, depicted in Table~\ref{supp: table: classifier_arch}, suggest that there is no notable advantage in using deeper and larger nets, and that smaller nets seem to perform just as good, with the added benefit of shorter training time. 

{\renewcommand{\arraystretch}{1.5}
\begin{table}[ht]

\begin{center}

\begin{tabular}{ c c c}

 \toprule

Architecture  & \# Params & Accuracy  \\
 \midrule

ResNet-20 & 0.27M & \ttc{59.03}{.58}    \\

ResNet-32 & 0.46M  & \ttc{58.4}{.62} \\
 
WRN-28-10  & 36.5M  &    \ttc{59.6}{.99} \\

 \bottomrule
\end{tabular}
\end{center}
\caption{Classification accuracy of different networks on a mixed dataset of images generated by our model and real images from CIFAR-10 with 100 samples per class}
\label{supp: table: classifier_arch}
\end{table}
}

\subsection{Generator Architecture}
\label{supp: gen_arch}
We tested many commonly used generative network architectures, and assesed their impact on the model's performance. The models that were evaluated are thus:
\begin{enumerate}
    \item \emph{\textbf{InfoGAN}} with transposed convolutions
    \item \emph{\textbf{DCGAN}}  with residual blocks and upscaling convolutions
    \item \emph{\textbf{CGAN}} with residual blocks, upscaling convolutions and conditional Batch-Norm 
\end{enumerate}
Note that we use the GAN models name for reference only, as our method uses these architectures in a non-adversarial approach, with no discriminator.
The differences of the above are summarized in Table~\ref{supp: arch_compare}, and the results are given in Table~\ref{supp: gen_arch_results}.
Implementation of all of the above was conducted according to \cite{lee2020mimicry}.

{\renewcommand{\arraystretch}{1.5}
\begin{table}[ht]

\begin{center}
\begin{tabular}{ c c c c c}

 \toprule

Arch  & \# Params & Residual & Upscaling & Batch-Norm   \\
 \midrule

\emph{InfoGAN} & 8.6M & \xmark &  \thead{transposed \\ convolution} & none \\

\emph{DCGAN} & 4.1M  & \cmark & \thead{bilinear \\ upsampling} & global \\
 
\emph{CGAN}  & 4.1M  & \cmark &   \thead{bilinear \\ upsampling} &   conditional  \\

 \bottomrule
\end{tabular}
\end{center}
\caption{Design differences between evaluated architectures}
\label{supp: arch_compare}
\end{table}
}

{\renewcommand{\arraystretch}{1.5}
\begin{table}[ht]

\begin{center}
\begin{tabular}{ c c c c c}

 \toprule

Architecture   & FID $\downarrow$ & CAS $\uparrow$ & CAS-Test   \\
 \midrule

\emph{InfoGAN}   & \ttc{49.38}{.32}  &  \ttc{67.65 }{.26} &  \ttc{85.14 }{.66}  \\

\emph{DCGAN} & \ttc{85.94}{2.14} & \ttc{61.48}{.23} & \ttc{45.14}{.74} \\
 
\emph{CGAN}  &\textbf{ \ttc{39.49}{.20}}&   \textbf{ \ttc{70.66}{.91}} &   \ttc{85.61}{.46}  \\

 \bottomrule
\end{tabular}
\end{center}
\caption{Results were obtained on sCOLA trained on the full train set of CIFAR-10. The conditional batch norm had a notable effect on the quality of the model's generations.}
\label{supp: gen_arch_results}
\end{table}
}

\section{Qualitative comparison}

\begin{figure*}[ht]
\vspace{-0.65cm}
	\caption{visualization of generations of CGAN (left) and sCOLA (right) trained on CIFAR-10 with varying samples per class (spc). While CGAN suffers from mode-collapse when training data drops, our method maintains a an intra-calss diversity that better matches the original data.}
		\label{supp:fig:cgan_vs_scola_qualitative}
		\vspace{-0.8cm}
	\centering
    \includegraphics{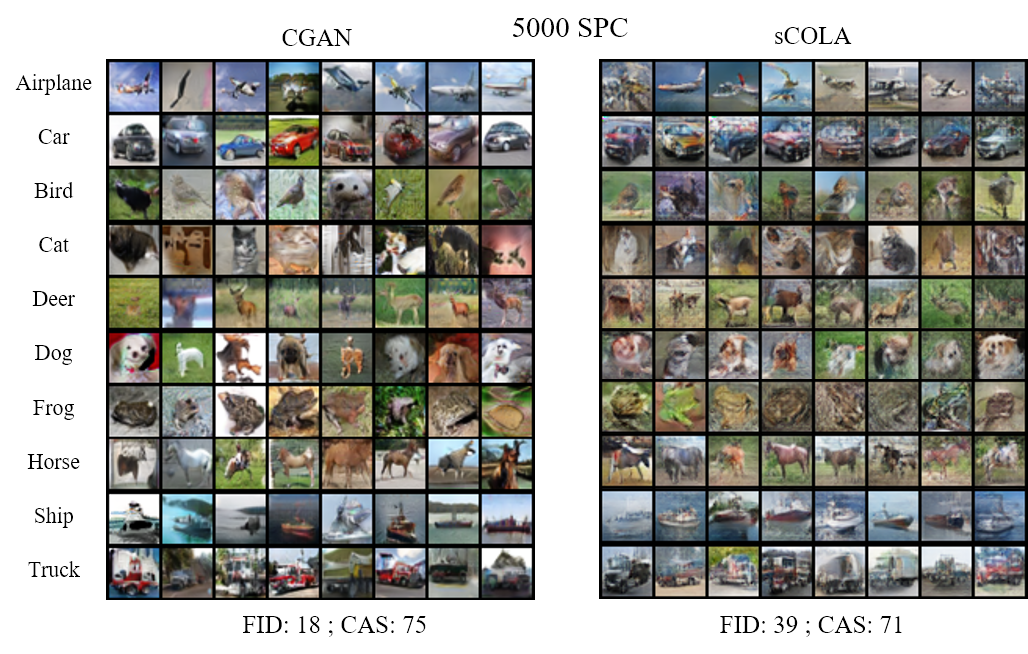} 
    
\end{figure*}

\begin{figure*}[ht]
\vspace{-0.5cm}

	\centering

    \includegraphics{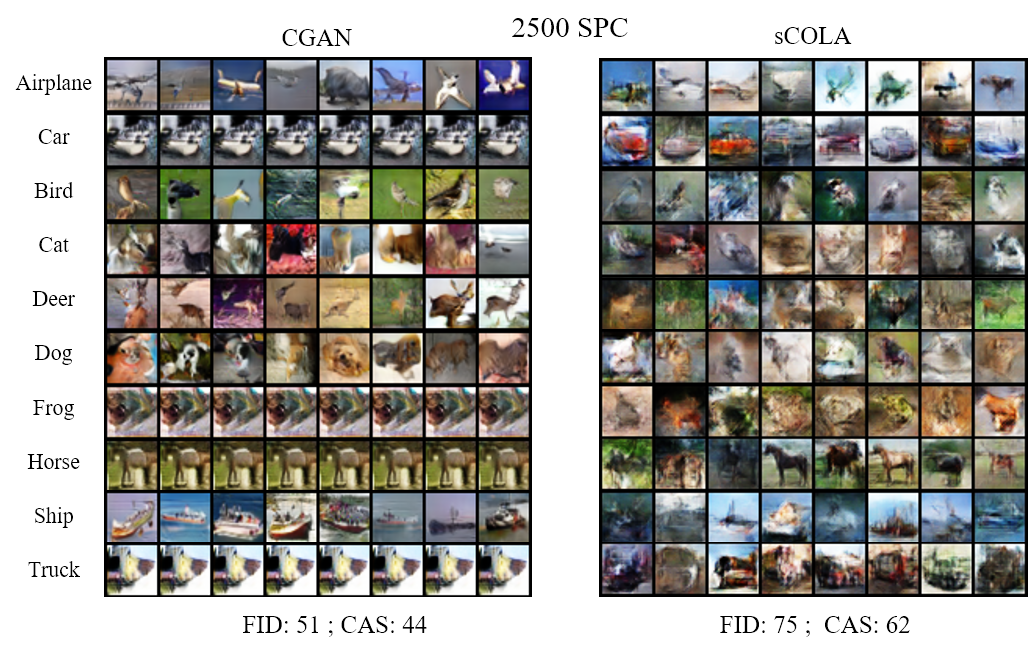} 
	
\end{figure*}
\begin{figure*}[ht]
	\centering

    \includegraphics{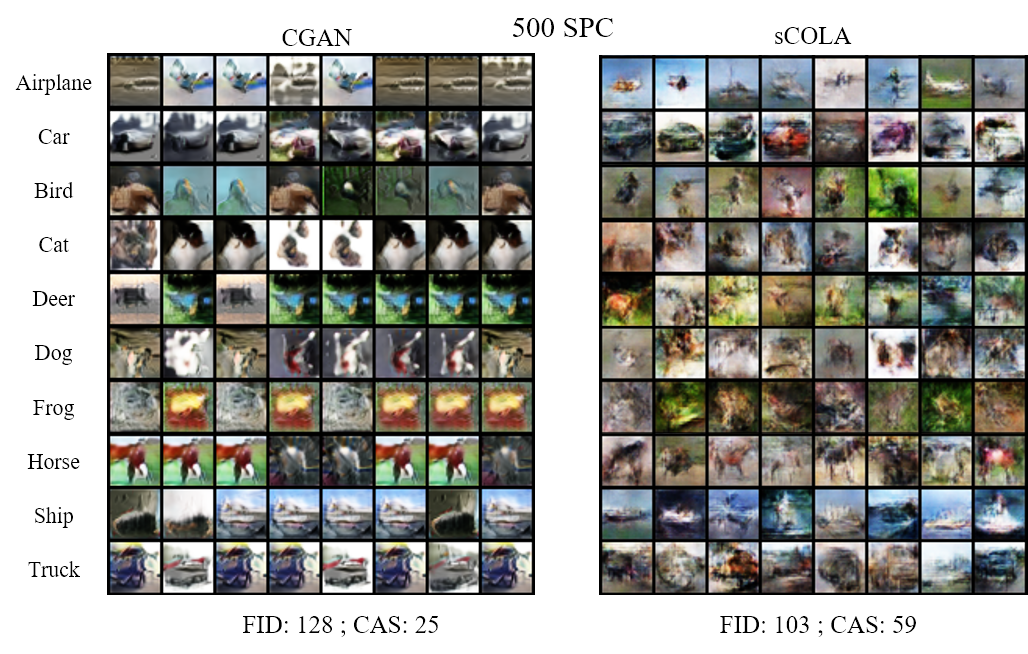} 
\end{figure*}

\begin{figure*}[ht]
\vspace{-0.5cm}
	\centering

    \includegraphics{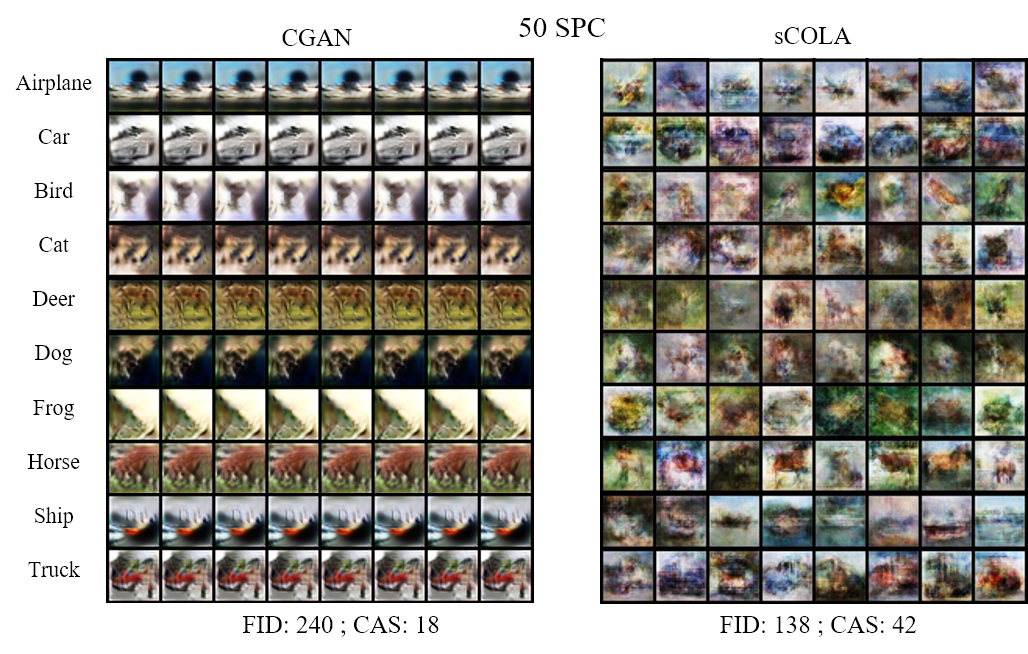}

\end{figure*}

\end{document}